\documentclass[runningheads]{llncs}

\usepackage{epsfig}
\usepackage{graphicx}
\usepackage{amsmath}
\usepackage{amssymb}

\usepackage{comment}
\usepackage{color}

\usepackage{algorithm}
\usepackage{algpseudocode}
 \makeatletter
 \let\OldStatex\Statex
 \renewcommand{\Statex}[1][0]{%
   \setlength\@tempdima{\algorithmicindent}%
  \OldStatex\hskip\dimexpr#1\@tempdima\relax}
 \makeatother
 
\algnewcommand\algorithmicinput{\textbf{Input:}}
\algnewcommand\Input{\item[\algorithmicinput]}
\algnewcommand\algorithmicoutput{\textbf{Output:}}
\algnewcommand\Output{\item[\algorithmicoutput]}

\usepackage[pagebackref=true,breaklinks=true,letterpaper=true,colorlinks,bookmarks=false]{hyperref}

\usepackage{relsize}

\usepackage{xcolor}
\usepackage{prettyref}
\usepackage{xspace}
\usepackage{bm}
\newrefformat{prob}{Problem\,\ref{#1}}
\newrefformat{def}{Definition\,\ref{#1}}
\newrefformat{sec}{Section\,\ref{#1}}
\newrefformat{sub}{Section\,\ref{#1}}
\newrefformat{prop}{Proposition\,\ref{#1}}
\newrefformat{app}{Appendix\,\ref{#1}}
\newrefformat{alg}{Algorithm\,\ref{#1}}
\newrefformat{cor}{Corollary\,\ref{#1}}
\newrefformat{thm}{Theorem\,\ref{#1}}
\newrefformat{lem}{Lemma\,\ref{#1}}
\newrefformat{fig}{Fig.\,\ref{#1}}
\newrefformat{tab}{Table\,\ref{#1}}

\newcommand{\bdmath}{\begin{dmath}}
\newcommand{\edmath}{\end{dmath}}
\newcommand{\beq}{\begin{equation}}
\newcommand{\eeq}{\end{equation}}
\newcommand{\bdm}{\begin{displaymath}}
\newcommand{\edm}{\end{displaymath}}
\newcommand{\bea}{\begin{eqnarray}}
\newcommand{\eea}{\end{eqnarray}}
\newcommand{\beal}{\beq \begin{array}{ll}}
\newcommand{\eeal}{\end{array} \eeq}
\newcommand{\beas}{\begin{eqnarray*}}
\newcommand{\eeas}{\end{eqnarray*}}
\newcommand{\ba}{\begin{array}}
\newcommand{\ea}{\end{array}}
\newcommand{\bit}{\begin{itemize}}
\newcommand{\eit}{\end{itemize}}
\newcommand{\ben}{\begin{enumerate}}
\newcommand{\een}{\end{enumerate}}

\newcommand{\calE}{{\cal E}}

\newcommand{\M}[1]{{\bm #1}} 
\renewcommand{\boldsymbol}[1]{{\bm #1}}

\newcommand{\hide}[1]{}

\newcommand{\hiddenText}{{\color{gray} hidden text.}}
\newcommand{\hideWithText}[1]{\hiddenText}

\newcommand{\normsq}[2]{\left\|#1\right\|^2_{#2}}

\newcommand{\tran}{^{\mathsf{T}}}

\newcommand{\trace}[1]{\mathrm{tr}\left(#1\right)}

\newcommand{\inv}{^{-1}}

\newcommand{\eye}{{\mathbf I}}

\newcommand{\Real}[1]{ { {\mathbb R}^{#1} } }

\newcommand{\SO}[1]{\ensuremath{\mathrm{SO}(#1)}\xspace}

\newcommand{\MA}{\M{A}}

\newcommand{\MC}{\M{C}}

\newcommand{\MK}{\M{K}}
\newcommand{\MG}{\M{G}}
\newcommand{\MM}{\M{M}}

\newcommand{\MP}{\M{P}}
\newcommand{\MQ}{\M{Q}}
\newcommand{\MU}{\M{U}}
\newcommand{\MR}{\M{R}}
\newcommand{\MS}{\M{S}}

\newcommand{\MV}{\M{V}}
\newcommand{\MF}{\M{F}}
\newcommand{\MH}{\M{H}}
\newcommand{\ML}{\M{L}}

\newcommand{\MX}{\M{X}}
\newcommand{\MY}{\M{Y}}

\newcommand{\MZ}{\M{Z}}
\newcommand{\MGamma}{\M{\Gamma}}
\newcommand{\MDelta}{\M{\Delta}}
 
\newcommand{\MOmega}{\M{\Omega}}

\newcommand{\vb}{\boldsymbol{b}}

\newcommand{\vdelta}{\boldsymbol{\delta}}
\newcommand{\vgamma}{\boldsymbol{\gamma}}

\newcommand{\vnu}{\boldsymbol{\nu}}

\newcommand{\sumalledges}{
     \displaystyle
     \sum_{(i,j) \in \calE}}

\newcommand{\blue}[1]{{\color{blue}#1}}

\newcommand{\linkToPdf}[1]{\href{#1}{\blue{(pdf)}}}
\newcommand{\linkToPpt}[1]{\href{#1}{\blue{(ppt)}}}
\newcommand{\linkToCode}[1]{\href{#1}{\blue{(code)}}}
\newcommand{\linkToWeb}[1]{\href{#1}{\blue{(web)}}}
\newcommand{\linkToVideo}[1]{\href{#1}{\blue{(video)}}}
\newcommand{\award}[1]{\xspace} 

\newcommand{\nrNodes}{n}
\newcommand{\barMR}{\bar{\MR}}

\newcommand{\simple}[2]{#1\in#2^\nrNodes} 
\newcommand{\rotSOdN}[1]{\simple{\MR}{\SO{#1}}}
\newcommand{\QSOdN}[1]{\simple{\MQ}{\SO{#1}}}

\newcommand{\MEAS}{\bar{\ML}}
\newcommand{\St}{\MS}
\newcommand{\lowrank}{p}

\newcommand{\myvec}{\text{vec}}

\newcommand{\Stiefel}[2]{\text{St}(#1,#2)}

\newcommand{\frob}{{\scriptsize \text{F}}}

\newcommand{\liehat}[1]{[#1]}
\newcommand{\exphat}[1]{e^{\liehat{#1}}}

\newcommand{\so}{\mathfrak{so}}

\newcommand{\vomega}{\boldsymbol{\omega}}

\newcommand{\st}{\textnormal{s.t.}\;}
\newcommand{\pMLE}{f_{\textnormal{MLE}}^*}
\newcommand{\pSDP}{f_{\textnormal{SDP}}^*}
\newcommand{\pSDPLR}{f_{\textnormal{SDPLR}}^*}

\DeclareMathOperator{\Diag}{Diag} 
\DeclareMathOperator{\Skew}{Skew}  
\DeclareMathOperator{\BDiag}{BlockDiag}
\DeclareMathOperator{\grad}{grad}
\DeclareMathOperator{\Hess}{Hess}

\makeatletter
\newcommand{\printfnsymbol}[1]{%
  \textsuperscript{\@fnsymbol{#1}}%
}
\makeatother

\begin{document}

\pagestyle{headings}
\mainmatter
\def\ECCVSubNumber{5714}  

\title{Shonan Rotation Averaging:\\
Global Optimality by Surfing $SO(p)^n$}
\titlerunning{Shonan Rotation Averaging}
\author{Frank Dellaert\thanks{Equal contribution}\inst{1}\orcidID{0000-0002-5532-3566} 
\and David M.\ Rosen\printfnsymbol{1}\inst{2}
\and Jing Wu\inst{1}
\and Robert Mahony\inst{3}\orcidID{0000-0002-7803-2868}
\and Luca Carlone\inst{2}
}
\authorrunning{F. Dellaert et al.}
\institute{Georgia Institute of Technology, Atlanta, GA 
\email{\{fd27,jingwu\}@gatech.edu}
\and Massachusetts Inst. of Technology, Cambridge, MA
\email{\{dmrosen,lcarlone\}@mit.edu}
\and Australian National University, Canberra, Australia
\email{Robert.Mahony@anu.edu.au}
}

\maketitle


\begin{abstract}
Shonan Rotation Averaging is a fast, simple, and elegant rotation averaging algorithm that is guaranteed to recover globally optimal solutions under mild assumptions on the measurement noise.
Our method employs semidefinite relaxation in order to recover provably globally optimal solutions of the rotation averaging problem.  
In contrast to prior work, we show how to solve large-scale instances of these relaxations using manifold minimization on (only slightly) higher-dimensional rotation manifolds, re-using existing high-performance (but \emph{local}) structure-from-motion pipelines.
Our method thus preserves the speed and scalability of current SFM methods, while recovering \emph{globally} optimal solutions.

\end{abstract}

\section{Introduction}

\emph{Rotation averaging} is the problem of estimating a set of $\nrNodes$ unknown orientations $\MR_1, \dotsc, \MR_n \in \SO{d}$ from noisy measurements $\barMR_{ij} \in \SO{d}$ of the  \emph{relative rotations}  $\MR_i\inv \MR_j$ between them \cite{Hartley13ijcv,Govindu2001Combining}.  This problem frequently arises in geometric reconstruction; in particular, it occurs as a sub-problem in bundle adjustment \cite{Triggs00,Agarwal10eccv}, structure from motion \cite{Schonberger2016Structure}, multi-camera rig calibration \cite{Pless03cvpr}, and sensor network localization \cite{Tron2008Distributed}.  The development of \emph{fast} and \emph{reliable} algorithms for solving the rotation averaging problem is therefore of great practical interest.

While there are numerous (inequivalent) ways of formalizing the rotation averaging problem in common use \cite{Hartley13ijcv}, unfortunately all of them share the common features of (a) \emph{high dimensionality}, due to the typically large number $\nrNodes$ of orientations $\MR_i$ to be estimated, and (b) \emph{non-convexity}, due to the non-convexity of the space of rotations itself.  In consequence, \emph{all} of these approaches lead to optimization problems that are computationally hard to solve in general.

In this work, we address rotation averaging using \emph{maximum likelihood estimation}, as this 
provides strong statistical guarantees on the quality of the resulting estimate \cite{Cover91book,Hartley00}.  
We consider the maximum likelihood estimation problem:
\bea
\max_{ \rotSOdN{d} } \sumalledges \kappa_{ij} \trace{\MR_i \barMR_{ij} \MR_j\tran},
\label{eq:RA0}
\eea 
where the $\kappa_{ij} \ge 0$ are concentration parameters for an assumed Langevin noise model~\cite{Boumal14ii,Carlone15iros-duality3D,Rosen16wafr-sesync}.  Our goal in this paper is to develop a \emph{fast} and \emph{scalable} optimization method that is capable of computing \emph{globally} optimal solutions of the rotation averaging problem \eqref{eq:RA0} in practice, despite its non-convexity. 

\label{sec:ShonanRotationAveraging}

We propose a new, straightforward algorithm, \emph{Shonan Rotation Averaging}, for finding \emph{globally optimal} solutions of problem \eqref{eq:RA0}.
At its core, our approach simply applies the standard Gauss-Newton or Levenberg-Marquardt methods to a \emph{sequence} of successively \emph{higher-dimensional} rotation averaging problems
\bea
\max_{ \QSOdN{\lowrank} }
\sumalledges \kappa_{ij} \trace{\MQ_i \MP \barMR_{ij} \MP\tran \MQ_j\tran},
\label{eq:Shonan}
\eea 
for increasing $\lowrank \ge d$. Note that the only difference between \eqref{eq:RA0} and \eqref{eq:Shonan} is the $\lowrank\times d$ projection matrix $\MP \triangleq [\eye_d;0]$, which adapts the measurement matrix $\barMR_{ij}$ to the higher-dimensional rotations $\MQ_i$.  
We start by running local optimization on \eqref{eq:Shonan} with $\lowrank=d$, and if this fails to produce a globally-optimal solution, we increase the dimension $\lowrank$ and try again.  Under mild conditions on the noise, we prove that this simple approach succeeds in recovering a \emph{globally} optimal solution of the rotation averaging problem \eqref{eq:RA0}.

A primary contribution of this work is to show how the fast optimization approach developed in \cite{Rosen16wafr-sesync,Rosen17irosws-SEsyncImplementation} can be adapted to run directly on the manifold of rotations (rather than the Stiefel manifold), implemented using the venerable Gauss-Newton or Levenberg-Marquardt methods. This approach enables Shonan Averaging to be easily implemented in high-performance optimization libraries commonly used in robotics and computer vision \cite{CeresManual,Kuemmerle11icra,Dellaert12tr}.



\section{Related Work}

By far the most common approach to addressing smooth optimization problems  in computer vision is to apply standard first- or second-order nonlinear optimization methods  to compute a critical point of the objective function \cite{Nocedal99}; this holds in particular for the rotation averaging problem  (see \cite{Hartley13ijcv} generally).  This approach is very attractive from the standpoint of computational efficiency, as the low per-iteration cost of these techniques (exploiting the sparsity present in real-world problems) enables these methods to scale gracefully to very large problem sizes; indeed, it is now possible to process reconstruction problems (of which rotation averaging is a crucial part) involving \emph{millions} of images on a single machine \cite{Frahm2010Building,Heinly2015Reconstructing}.  However, this computational efficiency comes at the expense of \emph{reliability}, as the use of \emph{local} optimization methods renders this approach susceptible to  convergence to bad (significantly suboptimal) local minima.

To address this potential pitfall, several recent lines of work have studied the convergence behavior of local search techniques applied to the rotation averaging problem.  One thrust proposes various initialization techniques that attempt to start the local search in low-cost regions of the state space, thereby favoring convergence to the true (global) minimum \cite{Martinec07cvpr,Carlone15icra-initPGO3D,Boumal2013robust,Singer10achm}.  Another direction investigates the size of the locally convex region around the global minimizer, in order to understand when local search is likely to succeed \cite{Hartley13ijcv,Wilson2016rotations}.  A third class of approaches attempts to evaluate the \emph{absolute} quality of a candidate solution $\hat{\MR}$ by employing Lagrangian duality to derive an \emph{upper bound} on $\hat{\MR}$'s (global) suboptimality \cite{Fredriksson12accv,Carlone15iros-duality3D}.  Interestingly, while these last two works employ different representations for rotations (\cite{Fredriksson12accv} uses quaternions, whereas \cite{Carlone15iros-duality3D} uses rotation matrices), \emph{both} of the resulting dual problems are semidefinite programs \cite{Vandenberghe96siam}, and \emph{both} of these duals are observed to be \emph{tight} unless the measurements $\barMR_{ij}$ are contaminated by large amounts of noise; this fact enables the \emph{certification} of optimality of a global minimizer $\MR^*$.

Motivated by the striking results reported in \cite{Fredriksson12accv,Carlone15iros-duality3D}, a recent line of work proposes to recover \emph{globally} optimal solutions of the rotation averaging problem from a solution of the Lagrangian dual.  Both \cite{Fredriksson12accv} and \cite{Carlone15iros-duality3D} propose to compute such solutions using an off-the-shelf SDP solver; however, as general-purpose SDP methods do not scale well with the problem size, this approach is limited to problems involving only a few hundred states.  More recently, \cite{Eriksson18cvpr-strongDuality,Eriksson19pami_duality} proposed a block-coordinate-descent method specifically tailored to the dual of \eqref{eq:RA0}, and showed that this method was 1-2 orders of magnitude faster than the standard SDP algorithm SeDuMi \cite{Sturm1999SeDuMi}; however, the reported results were still limited to problems involving at most $\approx300$ states.  Finally, \cite{Rosen16wafr-sesync} presents a global optimization approach for pose-graph SLAM based upon the \emph{dual} of the Lagrangian dual (also an SDP), together with a fast optimization scheme that is capable of solving problems involving tens of thousands of poses in a few seconds.  However, this optimization approach uses a truncated-Newton Riemannian optimization method \cite{Absil07book} employing an \emph{exact} model Hessian, and so cannot be deployed using the Gauss-Newton framework \cite{Nocedal99} that forms the basis of standard optimization libraries commonly used in computer vision applications \cite{CeresManual,Kuemmerle11icra,Dellaert12tr}.

Our method builds on the approach of \cite{Rosen16wafr-sesync}, but adapts the optimization scheme to run directly on the manifold of rotations using the Gauss-Newton method or a trust-region variant like Levenberg-Marquardt.
In this way, it is able to leverage the availability of high-performance software libraries \cite{CeresManual,Kuemmerle11icra,Dellaert12tr} for performing \emph{local} search on problems of the form \eqref{eq:Shonan} while preserving \emph{global optimality} guarantees.  
In addition, working with the rotation manifold $\SO{p}$, for $p\geq3$ avoids introducing new, unfamiliar objects like Stiefel manifolds in the core algorithm. The result is a simple, intuitive method for globally optimal rotation averaging that improves upon the scalability of current global methods \cite{Eriksson18cvpr-strongDuality} by an order of magnitude.

\section{Gauss-Newton for Rotation Averaging}
\label{sec:Synchronization}

This section reviews how the Gauss-Newton (GN)  algorithm is applied to find a first-order critical point of the rotation averaging problem \prettyref{eq:RA0}.
%
\label{GaussNewton_for_SOd}
We first rewrite \prettyref{eq:RA0} in terms of minimizing a sum of Frobenius norms:
\bea
\min_{\rotSOdN{d} } \sumalledges 
\kappa_{ij} \normsq{\MR_j  - \MR_i \barMR_{ij}}{\frob}.
\label{eq:GN0}
\eea
This can in turn be vectorized as:
\bea
\min_{ \rotSOdN{d} } \sumalledges 
\kappa_{ij} \normsq{ \myvec(\MR_j)  - \myvec(\MR_i \barMR_{ij}) }{2}
\label{eq:GN1a}
\eea
where ``$\myvec$'' is the column-wise vectorization of a matrix, and we made use of the fact that $\normsq{\MM}{\frob} = \normsq{\myvec(\MM)}{2}.$

Problem \eqref{eq:GN1a} does not admit a simple closed-form solution.  Therefore, given a feasible point $\MR = (\MR_1, \dotsc, \MR_\nrNodes) \in \SO{d}^n$, we will content ourselves with identifying a set of tangent vectors $\vomega_1, \dotsc, \vomega_n \in \so(d)$ along which we can \emph{locally perturb} each rotation $\MR_i$: 
\bea
\MR_i \leftarrow \MR_i \exphat{\vomega_i} \label{eq:retraction}
\eea
to \emph{decrease} the value of the objective; here $\liehat{\vomega_i}$ is the $d\times d$ skew-symmetric matrix corresponding to the hat operator of the Lie algebra $\so(d)$ associated with the rotation group $\SO{d}$.  We can therefore reformulate problem \eqref{eq:GN1a} as: 
\bea
\min_{\vomega \in \so(d)^\nrNodes} \sumalledges 
\kappa_{ij} \normsq{ \myvec (\MR_j\exphat{\vomega_j}) - \myvec(\MR_i\exphat{\vomega_i} \barMR_{ij})}{2}.
\label{eq:GN2a}
\eea
In effect, equation \eqref{eq:GN2a} replaces the optimization over the \emph{rotations} $\MR_i$ in \eqref{eq:GN1a} by an optimization over the \emph{tangent vectors} $\vomega_i \in \so(d)$. This is advantageous because $\so(d)$ is a \emph{linear} space, whereas $\SO{d}$ is not.  However, we still cannot solve \eqref{eq:GN2a}  directly, as the $\vomega_i$ enter the objective through the (nonlinear) exponential map.

However, we can \emph{locally approximate} the exponential map to first order as $\exphat{\vnu}\approx\eye+\liehat{\vnu}$. 
Therefore, for any matrix $\MA$ we have
\bea
\myvec(\barMR \exphat{\vnu} \MA)
&\approx& \myvec(\barMR (\eye+\liehat{\vnu}) \MA) \label{eq:exphatlin1}\\
&=& \myvec(\barMR \MA) + \myvec(\barMR \liehat{\vnu} \MA) \\
&=& \myvec(\barMR \MA) + (\MA\tran \otimes \barMR) \myvec(\liehat{\vnu}) \label{eq:exphatlin2},
\eea
where we made use of a well-known property of the Kronecker product $\otimes$.  We can also decompose the skew-symmetric matrix $\liehat{\vnu}$ in terms of the coordinates $\nu^k$ of the vector $\vnu$
according to:
\bea
\myvec(\liehat{\vnu})) = \myvec(\sum \vnu^k \MG_k) = \sum \vnu^k \myvec(\MG_k) = \bar{\MG}_d \vnu \label{eq:sod_generators}
\eea
where $\MG_k$ is the $k$th generator of the Lie algebra $\so(d)$, and $\bar{\MG}_d$ is the matrix obtained by concatenating the vectorized generators $\myvec(\MG_k)$ column-wise.

The (local) Gauss-Newton model of the rotation averaging problem \eqref{eq:RA0} is obtained by substituting the linearizations \eqref{eq:exphatlin1}--\eqref{eq:exphatlin2} and decomposition \eqref{eq:sod_generators} into \eqref{eq:GN2a} to obtain a \emph{linear least-squares} problem in the tangent vectors $\vomega_i$:
\bea 
\min_{\vomega \in \so(d)^\nrNodes} \sumalledges 
\kappa_{ij} \normsq{ \MF_j \vomega_j - \MH_i \vomega_i - \vb_{ij}  }{2},
\label{eq:GN5a}
\eea
where the Jacobians $\MF_j$ and $\MH_i$ and the right-hand side $\vb_{ij}$ can be calculated as
\bea
\MF_j & \doteq & (\eye \otimes \barMR_j) \bar{\MG}_d  \label{eq:GNlinearization1}\\
\MH_i & \doteq & (\barMR_{ij}\tran \otimes \barMR_i) \bar{\MG}_d  \\
\vb_{ij} & \doteq & \myvec (\barMR_i \barMR_{ij} - \barMR_j). \label{eq:GNlinearization2}
\eea

The \emph{local model} problem ~\eqref{eq:GN5a} can be solved efficiently to produce an optimal \emph{correction} $\vomega^\star \in \so(d)^\nrNodes$.   (To do so one can use either direct methods, based on sparse matrix factorization, or preconditioned conjugate gradient (PCG).  As an example, to produce the results in Section \ref{Experimental_results_section} we use PCG with a block-Jacobi preconditioner and a Levenberg-Marquardt trust-region method.)  This correction is then applied to \emph{update} the state $\MR$ as in equation \eqref{eq:retraction}. Typically, this is done in conjunction with a trust-region control strategy to prevent taking a step that leaves the neighborhood of $\MR$ in which the local linear models \eqref{eq:GN5a}--\eqref{eq:GNlinearization2} well-approximate the objective. The above process is then repeated to generate a \emph{sequence} $\lbrace \vomega \rbrace$ of such corrections, each of which improves the objective value, until some termination criterion is satisfied.

We emphasize that while the Gauss-Newton approach is sufficient to \emph{locally improve} an initial estimate, because of non-convexity the final iterate returned by this method is \emph{\textbf{not}} guaranteed to be a minimizer of \eqref{eq:RA0}.



\section{Shonan Rotation Averaging}



\subsection{A convex relaxation for rotation averaging}

The main idea behind Shonan Averaging is to develop a \emph{convex relaxation} of \eqref{eq:RA0} (which can be solved \emph{globally}), and then exploit this relaxation to search for good solutions of the rotation averaging problem.  Following \cite[Sec.\ 3]{Rosen16wafr-sesync}, in this section we derive a convex relaxation of \eqref{eq:RA0} whose minimizers in fact provide \emph{exact}, \emph{globally optimal} solutions of the rotation averaging problem subject to mild conditions on the measurement noise. 

To begin, we rewrite problem \eqref{eq:RA0} in a more compact, matricized form as:
\bea
\label{eq:matricized_rotation_averaging}
\pMLE = \min_{ \rotSOdN{d} } \trace{\MEAS \MR\tran \MR},
\label{eq:QRR}
\eea 
where $\MR = (\MR_1, \dotsc, \MR_n)$ is the $d\times d \nrNodes$ matrix of rotations $\MR_i\in\SO{d}$, and $\MEAS$ is a symmetric $(d \times d)$-block-structured matrix constructed from the measurements $\barMR_{ij}$.  The matrix $\MEAS$, known as the \emph{connection Laplacian} \cite{Singer2012Vector}, is the generalization of the standard (scalar) graph Laplacian to a graph having matrix-valued data $\barMR_{ij}$ assigned to its edges.

Note that $\MR$ enters the objective in \eqref{eq:QRR} only through the product $\MR\tran \MR$; this is a rank-$d$ positive-semidefinite matrix (since it is an outer product of the rank-$d$ matrix $\MR$), and has a $(d\times d)$-block-diagonal comprised entirely of identity matrices (since the blocks $\MR_i$ of $\MR$ are rotations).
Our convex relaxation of \eqref{eq:QRR}  is derived simply by replacing the rank-$d$ product $\MR\tran \MR$ with a \emph{generic} positive-semidefinite matrix $\MZ$ having identity matrices along its $(d \times d)$-block-diagonal:
\bea
\begin{gathered}
\pSDP = \min_{\MZ \succeq 0} \trace{\MEAS \MZ} \\
\st \BDiag_{d \times d}(\MZ) = (\eye_d, \dotsc, \eye_d).
\end{gathered}
\label{eq:SDP}
\eea 

Problem \eqref{eq:SDP} is a \emph{semidefinite program} (SDP) \cite{Vandenberghe96siam}: it requires the minimization of a linear function of a positive-semidefinite matrix $\MZ$, subject to a set of linear constraints.  Crucially, since the set of positive-semidefinite matrices is a convex cone, SDPs are \emph{always} convex, and can therefore be solved \emph{globally} in practice.  Moreover, since by construction \eqref{eq:QRR} and \eqref{eq:SDP} share the same objective, and the feasible set of \eqref{eq:SDP} contains every matrix of the form $\MR\tran \MR$ with $\MR$ feasible in \eqref{eq:QRR}, we can regard \eqref{eq:SDP} as a convexification of \eqref{eq:QRR} obtained by \emph{expanding the latter's feasible set}.  It follows immediately that $\pSDP \le \pMLE$.  Furthermore, if it so happens that a minimizer $\MZ^\star$ of \eqref{eq:SDP} admits a factorization of the form $\MZ^\star = {\MR^\star}\tran \MR^\star$ with $\MR^\star \in \SO{d}^n$, then it is clear that $\MR^\star$ is also a (global) minimizer of \eqref{eq:QRR}, since it attains the lower-bound $\pSDP$ for \eqref{eq:QRR}'s optimal value $\pMLE$.  The key fact that motivates our interest in the relaxation \eqref{eq:SDP} is that this favorable situation \emph{actually occurs}, provided that the noise on the observations $\barMR_{ij}$ is not too large.  More precisely, the following result is a specialization of \cite[Proposition 1]{Rosen16wafr-sesync} to the rotation averaging problem \eqref{eq:matricized_rotation_averaging}.

\begin{theorem}
\label{Sufficient_conditions_for_exactness_theorem}
Let ${\ML}$ denote the connection Laplacian for problem \eqref{eq:matricized_rotation_averaging} constructed from the true \emph{(}noiseless\emph{)} relative rotations $\MR_{ij} \triangleq \MR_i\inv \MR_j$.  Then there exists a constant $\beta \triangleq \beta(\ML)$ \emph{(}depending upon $\ML$\emph{)} such that, if $\lVert \MEAS - \ML\rVert_2 \le \beta$:
\begin{itemize}
\item[$(i)$]  The semidefinite relaxation \eqref{eq:SDP} has a unique solution $\MZ^\star$, and
\item[$(ii)$]  $\MZ^\star = {\MR^\star}\tran \MR^\star$, where $\MR^\star \in \SO{d}^n$ is a \emph{globally optimal} solution of the rotation averaging problem \eqref{eq:matricized_rotation_averaging}.
\end{itemize}
\label{SDP_exactness_theorem}
\end{theorem}


\subsection{Solving the semidefinite relaxation: The Riemannian Staircase}


In this section, we describe a specialized optimization procedure that enables the fast solution of large-scale instances of the semidefinite relaxation \eqref{eq:SDP}, following the approach of \cite[Sec.\ 4.1]{Rosen16wafr-sesync}. 

Theorem \ref{SDP_exactness_theorem} guarantees that in the (typical) case that \eqref{eq:SDP} is exact, the \emph{solution} $\MZ^\star$ that we seek admits a concise description in the factored form $\MZ^\star = {\MR^\star}\tran \MR^\star$ with $\MR^\star \in \SO{d}^n$. More generally, even in those cases where exactness fails, minimizers $\MZ^\star$ of \eqref{eq:SDP} generically have a rank $\lowrank$ not much greater than $d$, and therefore admit a symmetric rank decomposition $\MZ^\star = {\MS^\star}\tran \MS^\star$ for $\MS^\star \in \Real{\lowrank \times dn}$ with $\lowrank \ll dn$.  We exploit the existence of such low-rank solutions by adopting the approach of \cite{Burer03mp}, and replacing the decision variable $\MZ$ in \eqref{eq:SDP} with a symmetric rank-$\lowrank$ product of the form $\MS\tran\MS$.  This substitution has the effect of dramatically reducing the size of the search space, as well as rendering the positive-semidefiniteness constraint \emph{redundant}, since $\MS\tran \MS \succeq 0$ for \emph{any} $\MS$.  The resulting \emph{rank-restricted} version of \eqref{eq:SDP} is thus a standard \emph{nonlinear program} with decision variable the low-rank factor $\MS$:
\bea
\begin{gathered}
\pSDPLR(\lowrank) = \min_{\MS \in \Real{\lowrank \times dn}} \trace{\MEAS \MS\tran\MS} \\
\st \BDiag_{d \times d}(\MS\tran\MS) = (\eye_d, \dotsc, \eye_d).
\end{gathered}
\label{eq:LRSDP_NLP}
\eea 

The identity block constraints in \eqref{eq:LRSDP_NLP} are equivalent to the $\lowrank\times d$ block-columns $\St_i$ of $\St$ being orthonormal $d$-frames in $\Real{\lowrank}$. The set of all orthonormal $d$-frames in $\Real{\lowrank}$ is a matrix manifold, the \textbf{Stiefel manifold} $\Stiefel{d}{\lowrank}$ \cite{Absil07book}:
\bea
\Stiefel{d}{\lowrank} \doteq \{\MM \in\Real{\lowrank\times d} \mid \MM\tran\MM=\eye_d\}.
\label{eq:StiefelY}
\eea
We can therefore rewrite \eqref{eq:LRSDP_NLP} as a low-dimensional \emph{unconstrained} optimization over a product of $\nrNodes$ Stiefel manifolds:
\bea
\begin{gathered}
\pSDPLR(\lowrank) = \min_{\MS \in \Stiefel{d}{\lowrank}^\nrNodes} \trace{\MEAS \MS\tran\MS}.
\end{gathered}
\label{eq:Riemannian_LRSDP}
\eea 

Now, let us compare the rank-restricted relaxation \eqref{eq:Riemannian_LRSDP} with the original rotation averaging problem \eqref{eq:QRR} and its semidefinite relaxation \eqref{eq:SDP}.  Since any matrix $\MR_i$ in the (special) orthogonal group satisfies conditions \eqref{eq:StiefelY} with $\lowrank=d$, we have the set of inclusions:
\bea
\SO{d} \subset \mathrm{O}(d) = \Stiefel{d}{d} \subset \dotsb \subset \Stiefel{d}{\lowrank} \subset \dotsb
\eea
Therefore, we can regard the rank-restricted problems \eqref{eq:Riemannian_LRSDP} as a \emph{hierarchy} of relaxations (indexed by the rank parameter $\lowrank$)  of \eqref{eq:QRR} that are intermediate between \eqref{eq:QRR} and \eqref{eq:SDP} for $d < \lowrank < r$, where $r$ is the lowest rank of any solution of \eqref{eq:SDP}.  Indeed, if \eqref{eq:SDP} has a minimizer of rank $r$, then it is clear by construction that for $\lowrank \ge r$ any (global) minimizer $\MS^\star$ of \eqref{eq:Riemannian_LRSDP} corresponds to a minimizer $\MZ^\star = {\MS^\star}\tran\MS^\star$ of \eqref{eq:SDP}.  However, unlike \eqref{eq:SDP}, the rank-restricted problems \eqref{eq:Riemannian_LRSDP} are no longer convex, since we have reintroduced the (nonconvex) orthogonality constraints in \eqref{eq:StiefelY}.  It may therefore not be clear that anything has really been gained by relaxing problem \eqref{eq:QRR} to problem \eqref{eq:Riemannian_LRSDP}, since it seems that we may have just replaced one difficult (nonconvex) optimization problem with another. The key fact that justifies our interest in the rank-restricted relaxations \eqref{eq:Riemannian_LRSDP} is the following remarkable result \cite[Corollary 8]{Boumal16nips}:

\begin{theorem}
\label{BM_optimality_theorem_thm}
If $\MS^\star \in \Stiefel{d}{\lowrank}^\nrNodes$ is a rank-deficient second-order critical point of \eqref{eq:Riemannian_LRSDP}, then $\MS^\star$ is a \emph{global} minimizer of \eqref{eq:Riemannian_LRSDP}, and $\MZ^\star = {\MS^\star}\tran\MS^\star$ is a global minimizer of the semidefinite relaxation \eqref{eq:SDP}.
\label{global_optimality_thm}
\end{theorem}

This result immediately suggests the following simple algorithm (the \textbf{Riemannian staircase} \cite{Boumal16nips}) for recovering solutions $\MZ^\star$ of \eqref{eq:SDP} from \eqref{eq:Riemannian_LRSDP}: for some small relaxation rank $\lowrank\geq d$, find a second-order critical point $\St^\star$ of problem \eqref{eq:Riemannian_LRSDP} using a local search technique. If $\St^\star$ is rank-deficient, then Theorem \ref{BM_optimality_theorem_thm} proves that $\St^\star$ is a \emph{global} minimizer of \eqref{eq:Riemannian_LRSDP}, and $\MZ^\star = {\St^\star}\tran \St^\star$ is a solution of  \eqref{eq:SDP}. 
Otherwise, increase the rank parameter $\lowrank$ and try again. 
In the worst possible case, we might have to take $\lowrank$ as large as $dn + 1$ before finding a rank-deficient solution.
However, in practice typically only one or two ``stairs" suffice  -- just a \emph{bit} of extra room in \eqref{eq:Riemannian_LRSDP} vs. \eqref{eq:matricized_rotation_averaging} is all one needs!

Many popular optimization algorithms only guarantee convergence to \emph{first}-order critical points because they use only limited second-order information \cite{Nocedal99}.  
This is the case for the Gauss-Newton and Levenberg-Marquardt methods in particular, where the model Hessian is positive-semidefinite by construction.
Fortunately, there is a simple procedure that one can use to test the global optimality of a \emph{first}-order critical point $\St^\star$ of \eqref{eq:Riemannian_LRSDP}, and (if necessary) to construct a \emph{direction of descent} that we can use to ``nudge'' $\St^\star$ away from stationarity before restarting local optimization at the next level of the Staircase \cite{Boumal2015Riemannian}.

\begin{theorem}
Let $\St^\star \in \Stiefel{d}{\lowrank}^\nrNodes$ be a first-order critical point of \eqref{eq:Riemannian_LRSDP}, define
\begin{equation}
\MC \triangleq \MEAS - \frac{1}{2}\BDiag_{d\times d}\left(\MEAS {\St^\star}\tran \St^\star + {\St^\star}\tran \St^\star \MEAS \right),
\label{certificate_matrix_definition}
\end{equation}
and let $\lambda_{\textnormal{min}}$ be the minimum eigenvalue of $\MC$, with corresponding eigenvector $v_{\textnormal{min}} \in \Real{d\nrNodes}$.  Then:
\begin{itemize}
 \item [$(i)$]  If $\lambda_{\textnormal{min}} \ge 0$, then $\St^\star$ is a global minimizer of \eqref{eq:Riemannian_LRSDP}, and $\MZ^\star = {\St^\star}\tran \St^\star$ is a global minimizer of \eqref{eq:SDP}.
 \item [$(ii)$]  If $\lambda_{\textnormal{min}} < 0$, then the higher-dimensional lifting $
\St^{+} \triangleq 
\begin{bmatrix}
\St^\star ; 0
\end{bmatrix} \in \Stiefel{d}{\lowrank + 1}^\nrNodes
\label{Splus_definition}
$ of $\St^\star$ is a stationary point for the rank-restricted relaxation \eqref{eq:Riemannian_LRSDP} at the next level $\lowrank + 1$ of the Riemannian Staircase, and $
\dot{\St}^{+} \triangleq 
\begin{bmatrix}
0 ; 
v_{\textnormal{min}\tran}
\end{bmatrix} \in T_{\St^+}\left(\Stiefel{d}{\lowrank + 1}^\nrNodes \right)
\label{second_order_descent_direction_for_Stiefel_manifold}
$
is a second-order descent direction from $\St^+$.
\end{itemize}
\label{saddle_escape_theorem}
\end{theorem}

\begin{remark}[Geometric interpretation of Theorems \ref{SDP_exactness_theorem}--\ref{saddle_escape_theorem}]   With $\MC$ defined as in \eqref{certificate_matrix_definition}, part (i) of Theorem\ \ref{saddle_escape_theorem} is simply the standard necessary and sufficient conditions for $\MZ^\star = {\St^\star}\tran \St^\star$ to be a minimizer of \eqref{eq:SDP} \cite{Vandenberghe96siam}. $\MC$ \emph{also} gives the action of the Riemannian Hessian of \eqref{eq:Riemannian_LRSDP} on tangent vectors; therefore, if $\lambda_\textnormal{min}(\MC) < 0$, the corresponding direction of negative curvature $\dot{\St}^{+}$ defined in part (ii) of Theorem\ \ref{saddle_escape_theorem} provides a second-order descent direction from $\St^+$.  Theorem\ \ref{BM_optimality_theorem_thm} is an analogue of Theorem\ \ref{saddle_escape_theorem}(i).  Finally, Theorem\ \ref{SDP_exactness_theorem} follows from the fact that the certificate matrix $\MC$ depends continuously upon both the data $\MEAS$ and $\MS^\star$, and is \emph{always} valid for $\MS^\star = \MR^\star$ in the noiseless case; one can then employ a continuity argument to show that $\MC$ \emph{remains} a valid certificate for $\MZ^\star = {\MR^\star}\tran \MR^\star$ as a minimizer of \eqref{eq:SDP} for sufficiently small noise; see \cite[Appendix C]{Rosen18ijrr-sesync} for details.
\end{remark}


\subsection{Rounding the solution}

\begin{algorithm}[t]
\caption{Rounding procedure for solutions of \eqref{eq:Riemannian_LRSDP}}
\label{Rounding_algorithm}
\begin{algorithmic}[1]
\Input A minimizer $\St \in \Stiefel{d}{\lowrank}^\nrNodes$ of \eqref{eq:Riemannian_LRSDP}.
\Output A feasible point $\hat{\MR} \in \SO{d}^\nrNodes$ for \eqref{eq:RA0}.
\Function{RoundSolution}{$\St$}
\State Compute a rank-$d$ truncated SVD $\MU_d \M\varXi_d \MV_d\tran$ for $\St$  and assign $\hat{\MR} \leftarrow \M\varXi_d \MV_d\tran$.
\State Set $N_{+} \leftarrow \lvert \lbrace \hat{\MR}_i \mid \det(\hat{\MR}_i)  > 0 \rbrace \rvert$. \label{count_positive_determinants}
\If{$N_{+} < \lceil \frac{n}{2} \rceil$}
\State $\hat{\MR} \leftarrow \Diag(1, \dotsc, 1, -1) \hat{\MR}$.
\EndIf \label{end_of_determinant_reflection}
\For{$i = 1, \dotsc, n$}
\State Set $\hat{\MR}_i \leftarrow$ \Call{NearestRotation}{$\hat{\MR}_i$} (see \cite{Umeyama91pami})
\EndFor
\State \Return $\hat{\MR}$
\EndFunction
 \end{algorithmic}
\end{algorithm}

Algorithm \ref{Rounding_algorithm} provides a truncated SVD procedure to extract a feasible point $\hat{\MR} \in \SO{d}^\nrNodes$ for the original rotation averaging problem \eqref{eq:RA0} from the optimal factor $\St^\star \in \Stiefel{d}{\lowrank}^\nrNodes$ obtained via the rank-restricted relaxation \eqref{eq:Riemannian_LRSDP}. We need to ensure that $\hat{\MR}$ is a \emph{global minimizer} of \eqref{eq:RA0} whenever the semidefinite relaxation \eqref{eq:SDP} is \emph{exact}, and is at least an \emph{approximate} minimizer otherwise.  
The factor $\MR$ of a symmetric factorization $\MZ = \MR \tran \MR$ is only unique up to left-multiplication by some $\MA \in \mathrm{O}(d)$.  The purpose of lines \ref{count_positive_determinants}--\ref{end_of_determinant_reflection}  is to choose a representative $\hat{\MR}$ with a majority of $d \times d$ blocks $\hat{\MR}_i$ satisfying $\det(\hat{\MR}_i) > 0$, since these should \emph{all} be rotations in the event \eqref{eq:SDP} is exact~\cite[Sec.\ 4.2]{Rosen16wafr-sesync}.

\subsection{From Stiefel manifolds to rotations}

In this section we show how to reformulate the low-rank optimization \eqref{eq:Riemannian_LRSDP} as an optimization over the product $\SO{\lowrank}^\nrNodes$ of rotations of $\lowrank$-dimensional space.  This is convenient from an implementation standpoint, as affordances for performing optimization over rotations are a standard feature of many high-performance optimization libraries commonly used in robotics and computer vision \cite{CeresManual,Kuemmerle11icra,Dellaert12tr}.  

The main idea underlying our approach is the simple observation that since the columns of a rotation matrix $\MQ \in \SO{\lowrank}$ are orthonormal, then in particular the first $d$ columns $\MQ_{[1:d]}$ form an orthonormal $d$-frame in $\lowrank$-space, i.e., the submatrix $\MQ_{[1:d]}$ is itself an element of $\Stiefel{d}{\lowrank}$. Let us define the following projection, which simply extracts these first $d$ columns:
\begin{equation}
\begin{gathered}
\pi \colon \SO{\lowrank} \to \Stiefel{d}{\lowrank} \\
\pi(\MQ) = \MQ \MP
\end{gathered}
\label{rotation_to_Stiefel_projection}
\end{equation}
where $\MP = [\eye_d;  0]$ is the $\lowrank \times d$ projection matrix appearing in \eqref{eq:Shonan}.  It is easy to see that $\pi$ is surjective for any $\lowrank > d$:\footnote{In the case that $d = \lowrank$, $\Stiefel{\lowrank}{\lowrank} \supset \SO{\lowrank}$, and it is impossible for $\pi$ to be surjective.} given any element $\St = [s_1, \dotsc, s_d] \in \Stiefel{d}{\lowrank}$, we can construct a rotation $\MQ \in \pi\inv(\St)$ simply by extending $\lbrace s_1, \dotsc s_d \rbrace$ to an orthonormal basis $\lbrace s_1, \dotsc, s_d, v_1, \dotsc, v_{p-d} \rbrace$ for $\Real{\lowrank}$ using the Gram-Schmidt process, and (if necessary) multiplying the final element $v_{p-d}$ by $-1$ to ensure that this basis has a positive orientation; the matrix $\MQ = [s_1, \dotsc, s_d, v_1, \dotsc, v_{p-d}] \in \SO{\lowrank}$ whose columns are the elements of this basis is then a rotation satisfying $\pi(\MQ) = \St$.  Conversely, if $\MQ \in \pi\inv(\St)$, then by \eqref{rotation_to_Stiefel_projection} $\MQ = [\St, \MV]$ for some $\MV \in \Real{\lowrank \times  (\lowrank - d) }$; writing the orthogonality constraint $\MQ\tran\MQ = \eye_\lowrank$ in terms of this block decomposition then produces:
\begin{equation}
\MQ\tran \MQ =
\begin{bmatrix}
\St\tran\St & \St\tran \MV  \\
\MV\tran \St & \MV\tran \MV
\end{bmatrix} = 
\begin{bmatrix}
\eye_d & 0 \\
0 & \eye_{\lowrank - d}
\end{bmatrix}.
\label{block_decomposition_for_preimage_of_piS}
\end{equation}
It follows from \eqref{block_decomposition_for_preimage_of_piS} that the preimage of any $\St \in \Stiefel{d}{p}$ under the projection $\pi$ in \eqref{rotation_to_Stiefel_projection} is given explicitly by:
\begin{equation}
\pi\inv(\St) = \left \lbrace [\St, \MV] \mid \; \MV \in \Stiefel{\lowrank-d}{\lowrank}, \:
 \St\tran \MV = 0, \:
 \det([\St, \MV]) = +1
\right \rbrace. \label{preimage_of_Stiefel_element_under_pi}
\end{equation}
Equations \eqref{rotation_to_Stiefel_projection} and \eqref{preimage_of_Stiefel_element_under_pi} provide a means of representing $\Stiefel{d}{\lowrank}$ using $\SO{\lowrank}$: given any $\St \in \Stiefel{d}{\lowrank}$, we can represent it using one of the rotations in \eqref{preimage_of_Stiefel_element_under_pi} in which it appears as the first $d$ columns, and conversely, given any $\MQ \in \SO{\lowrank}$, we can extract its corresponding Stiefel manifold element using $\pi$.  We can extend this relation to \emph{products} of rotations and Stiefel manifolds in the natural way:
\begin{equation}
\begin{gathered}
\Pi \colon \SO{\lowrank}^\nrNodes \to \Stiefel{d}{\lowrank}^\nrNodes \\
\Pi(\MQ_1, \dotsc, \MQ_\nrNodes) = \left(\pi(\MQ_1), \dotsc, \pi(\MQ_\nrNodes) \right)
\end{gathered}
\label{product_projection}
\end{equation}
and $\Pi$ is likewise surjective for $\lowrank > d$.

The projection $\Pi$ enables us to ``pull back" the rank-restricted optimization \eqref{eq:Riemannian_LRSDP} on $\Stiefel{d}{\lowrank}^\nrNodes$ to an equivalent optimization problem on $\SO{\lowrank}^\nrNodes$.  Concretely, if $f \colon \Stiefel{d}{\lowrank}^\nrNodes \to \Real{}$ is the objective of \eqref{eq:Riemannian_LRSDP}, then we simply define the objective $\tilde{f}$ of our ``lifted" optimization over $\SO{\lowrank}^\nrNodes$ to be the pullback of $f$ via $\Pi$:
\begin{equation}
\begin{gathered}
\tilde{f} \colon \SO{\lowrank}^\nrNodes \to \Real{} \\
\tilde{f}(\MQ) = f \circ \Pi(\MQ).
\label{pullback_objective}
\end{gathered}
\end{equation}
Comparing \eqref{eq:Riemannian_LRSDP}, \eqref{product_projection}, and \eqref{pullback_objective} reveals that the pullback of the low-rank optimization \eqref{eq:Riemannian_LRSDP} to $\SO{\lowrank}^\nrNodes$ is exactly the Shonan Averaging problem \eqref{eq:Shonan}.


\begin{theorem}
\label{equivalence_of_lifted_optimization_theorem}
Let $\lowrank > d$.  Then:
\begin{itemize}
\item[$(i)$] The rank-restricted optimization \eqref{eq:Riemannian_LRSDP} and the Shonan Averaging problem  \eqref{eq:Shonan} attain the same optimal value.
\item[$(ii)$] $\MQ^\star \in \SO{p}^\nrNodes$ minimizes \eqref{eq:Shonan} if and only if $\St^\star = \Pi(\MQ^\star)$ minimizes \eqref{eq:Riemannian_LRSDP}.
\end{itemize}
\end{theorem}

Similarly, we also have the following analogue of Theorem \ref{saddle_escape_theorem} for problem \eqref{eq:Shonan}:

\begin{theorem}
\label{certifying_optimality_of_lifted_optimization_theorem}
Let $\MQ^\star = (\MQ_1^\star, \dotsc, \MQ_\nrNodes^\star) \in \SO{\lowrank}^\nrNodes$ be a first-order critical point of problem \eqref{eq:Shonan}, and $\St^\star = (\St_1^\star, \dotsc, \St_\nrNodes^\star) = \Pi(\MQ^\star) \in \Stiefel{d}{\lowrank}^\nrNodes$.  Then:
\begin{itemize}
 \item [$(i)$]  $\St^\star$ is a first-order critical point of the rank-restricted optimization \eqref{eq:Riemannian_LRSDP}.
 \item [$(ii)$]  Let $\MC$ be the certificate matrix defined in \eqref{certificate_matrix_definition}, $\lambda_{\textnormal{min}}$ its minimum eigenvalue, and $v = (v_1, \dotsc, v_\nrNodes) \in \Real{d\nrNodes}$ a corresponding eigenvector.  If $\lambda_{\min} < 0$, then the point $\MQ^+ \in \SO{\lowrank + 1}^\nrNodes$ defined elementwise by:
 \begin{equation}
\MQ_i^+ = 
\begin{bmatrix}
\MQ^\star_i & 0 \\
0 & 1
\end{bmatrix} \in \SO{\lowrank + 1}
\label{higher_dimensional_rotation}
 \end{equation}
is a first-order critical point of problem \eqref{eq:Shonan} in dimension $\lowrank + 1$, and the tangent vector $\dot{\MQ}^+ \in T_{\MQ^+}(\SO{\lowrank+1}^\nrNodes)$ defined blockwise by:
\begin{equation}
\dot{\MQ}_i^+ = 
\MQ^+_i 
\begin{bmatrix}
0 & - v_i \\
v_i\tran & 0
\end{bmatrix} \in T_{\MQ_i^+}(\SO{\lowrank+1})
\label{second_order_descent_direction_for_rotation_manifold}
\end{equation}
is a second-order descent direction from $\MQ^+$.
 \end{itemize}
\end{theorem}

Theorems \ref{equivalence_of_lifted_optimization_theorem} and \ref{certifying_optimality_of_lifted_optimization_theorem} are proved in Appendix \ref{Proof_of_lifted_certifying_optimality_theorem_section} of the supplementary material.


\subsection{The complete algorithm}

Combining the results of the previous sections gives the complete \emph{Shonan Averaging} algorithm (Algorithm \ref{Shonan_Averaging}).  (We employ Levenberg-Marquardt to perform the fast local optimization required in line \ref{local_search_in_Shonan_averaging} -- see Appendix \ref{GN_for_Shonan_Averaging_section} for details).

\begin{algorithm}[t]
\caption{Shonan Rotation Averaging}
\label{Shonan_Averaging}
\begin{algorithmic}[1]
\Input  An initial point $\MQ \in \SO{\lowrank}^\nrNodes$ for \eqref{eq:Shonan},  $\lowrank \ge d$.
\Output A feasible estimate $\hat{\MR}\in \SO{d}^\nrNodes$ for the rotation averaging problem \eqref{eq:RA0}, and the lower bound $\pSDP$ for problem \eqref{eq:RA0}'s optimal value.
\Function{ShonanAveraging}{$\MQ$}
\Repeat  \Comment{Riemannian Staircase}
\State $\MQ \leftarrow \Call{LocalOpt}{\MQ}$  \Comment{Find critical point of \eqref{eq:Shonan}} \label{local_search_in_Shonan_averaging}
\State Set $\St \leftarrow \Pi(\MQ)$  \Comment{Project to $\Stiefel{d}{\lowrank}^\nrNodes$}
\State Construct the certificate matrix $\MC$ in \eqref{certificate_matrix_definition}, and \label{minimum_eigenvalue_computation}
\Statex[2] compute its minimum eigenpair $(\lambda_{\textnormal{min}}, v_{\textnormal{min}})$. 
\If{$\lambda_{\textnormal{min}} < 0$} \Comment{$\MZ = \St\tran \St$ is \emph{not} optimal}
\State Set $\lowrank \leftarrow \lowrank + 1$ and $
{\MQ}_i \leftarrow 
\begin{pmatrix}
\MQ_i & 0 \\
0 & 1
\end{pmatrix}$
 $\forall i$. \Comment{Ascend to next level}
\State Construct descent direction $\dot{\MQ}$ as in \eqref{second_order_descent_direction_for_rotation_manifold}.
\State $\MQ \leftarrow \Call{LineSearch}{\MQ, \dot{\MQ}}$  \Comment{Nudge $\MQ$} \label{Saddle_escape}
\EndIf

\Until{$\lambda_{\textnormal{min}} \ge 0$}  \Comment{$\MZ = \St\tran\St$ solves \eqref{eq:SDP}}
\State Set $\pSDP \leftarrow \trace{\MEAS  \St\tran\St}$.  
\State Set $\hat{\MR} \leftarrow \Call{RoundSolution}{\St}$.  
\State \Return $\left \lbrace  \hat{\MR}, \pSDP \right \rbrace$
\EndFunction
 \end{algorithmic}
\end{algorithm}

When applied to an instance of rotation averaging, Shonan Averaging returns a feasible point $\hat{\MR} \in \SO{d}^\nrNodes$ for the maximum likelihood estimation \eqref{eq:RA0}, and a lower bound $\pSDP \le \pMLE$ for problem \eqref{eq:RA0}'s optimal value.  This lower bound provides an \emph{upper} bound on the suboptimality of \emph{any} feasible point $\MR \in \SO{d}^\nrNodes$ as a solution of problem \eqref{eq:RA0} according to:
\begin{equation}
 \label{suboptimality_bound_for_Problem_1}
 f(\MEAS \MR\tran \MR) - \pSDP \ge f(\MEAS \MR\tran \MR) - \pMLE.
\end{equation}
Furthermore, if the relaxation \eqref{eq:SDP} is exact, the  estimate $\hat{\MR}$ returned by Algorithm \ref{Shonan_Averaging} \emph{attains} this lower bound:
\begin{equation}
 \label{MLE_attains_lower_bound_in_case_of_exactness}
  f(\MEAS \hat{\MR}\tran \hat{\MR}) = \pSDP.
\end{equation}
Consequently, verifying \emph{a posteriori} that \eqref{MLE_attains_lower_bound_in_case_of_exactness} holds provides a \emph{certificate} of $\hat{\MR}$'s \emph{global} optimality as a solution of the rotation averaging problem \eqref{eq:RA0}.


\section{Experimental Results}
\label{Experimental_results_section}

\begin{table*}[t!]
\footnotesize
\centerline{
\begin{tabular}{|c || c | c || c | c || c | c  || c | c ||}
\hline
& & & \multicolumn{5}{|c||}{Shonan Averaging} \\
&  $n$ & $m$ & $\lowrank$ & $\lambda_{\textnormal{min}}$ & $\pSDP$  & Opt.\ time [s] & Min.\ eig.\ time [s] \\
\hline 
smallGrid & $125$ & $297$ & $5$ & $-9.048\times 10^{-6}$  & $4.850 \times 10^2$ & $2.561 \times 10^{-1}$ & $2.613 \times 10^{-1}$ \\
\hline
sphere & $2500$ & $4949$ & $6$ & $-1.679\times 10^{-4}$ & $5.024 \times 10^0$  & $1.478 \times 10^1$ &  $1.593 \times 10^1$ \\
\hline
torus & $5000$ & $9048$ & $5$ & $-2.560 \times 10^{-4}$ &   $1.219 \times 10^4$ &  $7.819 \times 10^1$ & $8.198 \times 10^1$\\ 
\hline
garage & $1661$ & $6275$ & $5$ & $1.176 \times 10^{-4}$ &  $2.043 \times 10^{-1}$ & $1.351 \times 10^1$ & $1.420 \times 10^1$ \\ 
\hline 
cubicle & $5750$ & $16869$ & $5$ & $-9.588 \times 10^{-5}$ &  $3.129 \times 10^1$ & $1.148 \times 10^2$ & $1.208 \times 10^2$ \\
\hline
\end{tabular}
}
\caption{Shonan Averaging results for the SLAM benchmark datasets} 
\label{SLAM_benchmarks_table}
\end{table*}

In this section we evaluate Shonan Averaging's performance on 
(a) several large-scale problems derived from standard pose-graph SLAM benchmarks,
(b) small randomly generated synthetic datasets, 
and (c) several structure from motion (SFM) datasets used in earlier work. 
The experiments were performed on a 
desktop computer with a 6-core Intel i5-9600K CPU @ 3.70GHz running Ubuntu 18.04.  
Our implementation of Shonan Averaging (SA) is written in C++, using the GTSAM library \cite{Dellaert12tr} to perform the local optimization.
The fast minimum-eigenvalue computation described in \cite[Sec.\ III-C]{Rosen17irosws-SEsyncImplementation} is implemented using the symmetric Lanczos algorithm from the Spectra library\footnote{\url{https://Spectralib.org/}} to compute the minimum eigenvalue required in line \ref{minimum_eigenvalue_computation}.  
We initialize the algorithm with a randomly-sampled point $\MQ \in \SO{\lowrank}^\nrNodes$ at level $\lowrank = 5$, and require that the minimum eigenvalue satisfy $\lambda_{\textnormal{min}} \ge -10^{-4}$ at termination of the Riemannian Staircase.

The results for the pose-graph SLAM benchmarks demonstrate that Shonan Averaging is already capable of recovering \emph{certifiably globally optimal} solutions of large-scale real-world rotation averaging problems in tractable time.  
For  these experiments, we have extracted the rotation averaging problem \eqref{eq:RA0} obtained by simply ignoring the translational parts of the measurements. 
Table \ref{SLAM_benchmarks_table} reports the size of each problem (the number of unknown rotations $n$ and relative measurements $m$), together with the relaxation rank $\lowrank$ at the terminal level of the Riemannian Staircase, the optimal value of the semidefinite relaxation \eqref{eq:SDP} corresponding to the recovered low-rank factor $\St^\star$, and the total elapsed computation times for the local optimizations in line \ref{local_search_in_Shonan_averaging} and minimum-eigenvalue computations in line \ref{minimum_eigenvalue_computation} of Algorithm \ref{Shonan_Averaging}, respectively. 
We remark that the datasets torus in Table \ref{SLAM_benchmarks_table} is \emph{1-2 orders of magnitude larger} than the examples reported in previous work on direct globally-optimal rotation averaging algorithms \cite{Eriksson18cvpr-strongDuality,Fredriksson12accv}.

We attribute Shonan Averaging's improved scalability to its use of superlinear \emph{local} optimization in conjunction with the Riemanian Staircase; in effect, this strategy provides a means of ``upgrading'' a fast \emph{local} method to a fast \emph{global} one.  This enables Shonan Averaging to leverage existing heavily-optimized, scalable, and high-performance software libraries for  superlinear \emph{local} optimization \cite{CeresManual,Kuemmerle11icra,Dellaert12tr} as the main computational engine of the algorithm (cf.\ line \ref{local_search_in_Shonan_averaging} of Algorithm \ref{Shonan_Averaging}), while preserving the guarantee of \emph{global} optimality.  
\begin{table}[h]
\footnotesize
\centerline{
\begin{tabular}{clrrr}
\hline
     n,$\sigma$      & method   &   error &   time (s) &   success \\
\hline
 \hline
n=20 $\sigma$=0.2  & SA       &  0.000\% &      0.030 &      100\% \\
              & BD       &  1.397\% &      0.374 &       60\% \\
              & LM       & -0.001\% &      0.003 &       80\% \\
 \hline
n=20 $\sigma$=0.5  & SA       &  0.000\% &      0.037 &      100\% \\
              & BD       &  0.014\% &      0.446 &       40\% \\
              & LM       & -0.001\% &      0.004 &       20\% \\
 \hline
n=50 $\sigma$=0.2  & SA       &  0.000\% &      0.141 &      100\% \\
              & BD       &  0.291\% &      6.261 &       60\% \\
              & LM       & -0.039\% &      0.011 &       20\% \\
 \hline
n=50 $\sigma$=0.5  & SA       &  0.000\% &      0.194 &      100\% \\
              & BD       & -7.472\% &      6.823 &       60\% \\
              & LM       & -0.008\% &      0.014 &       40\% \\
 \hline
n=100 $\sigma$=0.2 & SA       &  0.000\% &      0.403 &      100\% \\
              & BD       &    nan\% &    nan     &        0\% \\
              & LM       & -0.014\% &      0.016 &       20\% \\
 \hline
n=100 $\sigma$=0.5 & SA       &  0.000\% &      0.373 &      100\% \\
              & BD       &  3.912\% &     32.982 &       20\% \\
              & LM       & -0.142\% &      0.023 &       60\% \\
 \hline
n=200 $\sigma$=0.2 & SA       &  0.000\% &      1.275 &      100\% \\
              & BD       &    nan\% &    nan     &        0\% \\
              & LM       & -0.102\% &      0.049 &       60\% \\
 \hline
n=200 $\sigma$=0.5 & SA       &  0.000\% &      1.845 &      100\% \\
              & BD       &    nan\% &    nan     &        0\% \\
              & LM       & -0.293\% &      0.046 &       40\% \\
\hline
\end{tabular}
}
\caption{Synthetic results for varying problem sizes $n$ and noise levels $\sigma$. SA=Shonan, BD=block coordinate descent from~\cite{Eriksson18cvpr-strongDuality}, LM=Levenberg-Marquardt.} 
\label{tab:synthetic}
\end{table}

The synthetic results confirm that Shonan Averaging always converges to the true (\emph{global}) minimizer, and significantly outperforms the block coordinate descent method from~\cite{Eriksson18cvpr-strongDuality,Eriksson19pami_duality} as soon as the number of unknown rotations becomes large. 
We followed the approach for generating the data from \cite{Eriksson18cvpr-strongDuality}, generating 4 sets of problem instances of increasing size, for two different noise levels. We generated 3D poses on a circular trajectory, forming a cycle graph. Relative rotation measurements were corrupted by composing with a random axis-angle perturbation, where the axis was chosen randomly and the angle was generated from a normal distribution with a standard deviation $\sigma=0.2$ or $\sigma=0.5$. 
Initialization for all three algorithms was done randomly, with angles uniform random over the range $(-\pi,\pi)$.
We compare our results with two baselines: (LM) Levenberg-Marquardt on the $\SO{3}^n$ manifold, also implemented in GTSAM, which is fast but not expected to find the global minimizer; (BD): the block coordinate descent method from~\cite{Eriksson18cvpr-strongDuality,Eriksson19pami_duality}. Because no implementation was available, we re-implemented it in Python, using a $3\times 3$ SVD decomposition at the core. However, our implementation differs in that we establish optimality using the eigenvalue certificate from Shonan.

The synthetic results are shown in Table \ref{tab:synthetic}. In the table we show the percentage of the cases in which either method found a global minimizer, either as certified by the minimum eigenvalue $\lambda_{\textnormal{min}} \ge -10^{-4}$, or being within $5\%$ of the optimal cost. Shonan Averaging (``SA" in the table) finds an optimal solution every time as certified by $\lambda_{\textnormal{min}}$. Levenberg-Marquardt is fast but finds global minima in only about 40\% to 60\% of the cases. Finally, the block coordinate descent method (BD) is substantially slower, and does not always converge within the allotted time. We found that in practice BD converges very slowly and it takes a long time for the algorithm to converge to the same value as found by SA or LM. In the table we have shown the relative error and running time (in seconds), both averaged over 5 runs. For $n=100$ and $n=200$ the BD method did often not converge to the global optimum within reasonable time.

\begin{table}[t]
\footnotesize
\centerline{
\begin{tabular}{lcrrr}
\hline
                     dataset & method   &   error &   time &   success \\
\hline
 reichstag            & SA       &  0.000\% &      0.197 &      100\% \\
 (n=71, m=2554)         & BD       &  0.000\% &      0.247 &      100\% \\
                      & LM       &  0.001\% &      0.085 &      80\% \\
\hline
 pantheon\_interior & SA       &  0.000\% &      0.971 &      100\% \\
 (n=186, m=10000)     & BD       &  0.000\% &     1.823 &      100\% \\
                      & LM       &  0.001\% &     0.313 &       100\% \\
\hline
\end{tabular}
}
\caption{Results on SFM problems from the YFCC dataset.} 
\label{tab:YFCC}
\end{table}

Finally, Table \ref{tab:YFCC} shows results on two datasets derived by Heinly et al.~\cite{Heinly15cvpr_ReconstructWorld6Days} from the large-scale YFCC-100M dataset~\cite{Thomee16acm_YFCC100M,Heinly15cvpr_ReconstructWorld6Days}. The relative measurements for these were derived from the SFM solution provided with the data, and corrupted with noise as before, using $\sigma=0.2$. All three methods agree on the globally optimal solution, and the timing data shows the same trend as in Table \ref{tab:synthetic}.



\section{Conclusion}

In this work we presented Shonan Rotation Averaging, a fast, simple, and elegant algorithm for rotation averaging that is \emph{guaranteed} to recover globally optimal solutions under mild conditions on the measurement noise.  Our approach applies a fast \emph{local} search technique to a \emph{sequence} of higher-dimensional lifts of the rotation averaging problem until a globally optimal solution is found.  Shonan Averaging thus leverages the speed and scalability of existing high-performance \emph{local} optimization methods already in common use, while enabling the recovery of \emph{provably optimal} solutions.

\newpage
{\small
\bibliographystyle{ieee_fullname}
\bibliography{refs,myRefs,FDrefs}
}

\clearpage

\appendix
{\Large Supplementary material}

\section{Proof of Theorems \ref{equivalence_of_lifted_optimization_theorem} and \ref{certifying_optimality_of_lifted_optimization_theorem} }
\label{Proof_of_lifted_certifying_optimality_theorem_section}
\subsection{Proof of Theorem \ref{equivalence_of_lifted_optimization_theorem} }

\begin{proof}
For part (i), let $\MQ^\star \in \SO{p}^\nrNodes$ be a minimizer of \eqref{eq:Shonan}, with corresponding optimal value $\tilde{f}^\star$, and suppose for contradiction that $\St^\star = \Pi(\MQ^\star)$ is \emph{not} a minimizer of \eqref{eq:Riemannian_LRSDP}.  Then there exists some $\St' \in \Stiefel{d}{\lowrank}^\nrNodes$ with $f(\St') < f(\St^\star) =  \tilde{f}^\star$.  However, since $\Pi$ is surjective for $\lowrank > d$, then there exists some $\MQ' \in \Pi\inv(\St')$, and by definition of $\tilde{f}$ in \eqref{pullback_objective} we have  that $\tilde{f}(\MQ') = f(\St') < \tilde{f}^\star$, contradicting the optimality of $\MQ^\star$.  We conclude that $\St^\star$ is indeed a minimizer of $f$ over $\Stiefel{d}{\lowrank}^\nrNodes$, and therefore the optimal values of \eqref{eq:Riemannian_LRSDP} and \eqref{eq:Shonan} coincide.

Part (ii) follows immediately from part (i) since $\tilde{f}(\MQ) = f(\St)$ for any $\MQ \in \SO{\lowrank}^\nrNodes$ and $\St = \Pi(\MQ)$ by \eqref{pullback_objective}.
\end{proof}

\subsection{Proof of Theorem \ref{certifying_optimality_of_lifted_optimization_theorem}}

In this section we derive Theorem \ref{certifying_optimality_of_lifted_optimization_theorem} as a consequence of Theorem \ref{saddle_escape_theorem} and the equivalence of the rank-restricted optimization  \eqref{eq:Riemannian_LRSDP} and the Shonan Averaging problem \eqref{eq:Shonan} (Theorem \ref{equivalence_of_lifted_optimization_theorem}).  To do so, we need to understand how the local geometry of the lifted objective $\tilde{f}$ relates to that of $f$ near a critical point $\MQ^\star \in \SO{\lowrank}^\nrNodes$.  Recall that the tangent spaces of the rotational and Stiefel manifolds can be expressed as:
\begin{equation}
T_\MQ(\SO{\lowrank}) = \left \lbrace \MQ \dot{\MDelta} \mid \dot{\MDelta} \in \Skew(\lowrank) \right \rbrace
\label{rotation_manifold_tangent_spaces}
\end{equation}
and
\begin{equation}
T_{\St}(\Stiefel{d}{\lowrank}) = \left \lbrace \St \dot{\MOmega} + \MV \dot{\MK} \mid
\begin{aligned}
&\MOmega \in \Skew(d), \\
&\dot{\MK} \in \Real{(\lowrank - d) \times d}, \\  
&\MV \in \Real{\lowrank \times (\lowrank -d)}, \\
&\St\tran \MV = 0 
\end{aligned}\right \rbrace
\label{Stiefel_manifold_tangent_spaces}
\end{equation}
respectively (cf.\ Example 3.5.2 of \cite{Absil07book}).  If we rewrite the elements of \eqref{rotation_manifold_tangent_spaces} in a block-partitioned form compatible with the action of the projection \eqref{rotation_to_Stiefel_projection}:
\begin{equation}
T_\MQ(\SO{\lowrank}) = \left \lbrace \MQ
\begin{bmatrix}
\dot{\MOmega} & -\dot{\MK}\tran \\
\dot{\MK} & \dot{\MGamma} 
\end{bmatrix}
\mid 
\begin{aligned}
\dot{\MOmega} &\in \Skew(d), \\
\dot{\MK} &\in \Real{(\lowrank -d) \times d}, \\
\dot{\MGamma} &\in \Skew(\lowrank-d)
\end{aligned}
\right \rbrace
\label{partitioned_rotation_manifold_tangent_spaces}
\end{equation}
then the derivative of the projection $\pi$ at $\MQ = \begin{bmatrix} \St & \MV \end{bmatrix} \in \SO{\lowrank}$ is:
\begin{equation}
\begin{gathered}
d\pi_{\MQ} \colon T_\MQ(\SO{\lowrank}) \to T_{\St}(\Stiefel{d}{\lowrank}) \\
d\pi_{\MQ}\left(
\MQ  
\begin{bmatrix}
\dot{\MOmega} & -\dot{\MK}\tran \\
\dot{\MK} & \dot{\MGamma} 
\end{bmatrix}
\right) 
= \St \dot{\MOmega} + \MV\dot{\MK}.
\end{gathered}
\label{derivative_of_rotation_to_Stiefel_projection}
\end{equation}

The next result is a direct consequence of \eqref{rotation_manifold_tangent_spaces}--\eqref{derivative_of_rotation_to_Stiefel_projection}:
\begin{lemma}
The projection $\pi \colon \SO{\lowrank} \to \Stiefel{d}{\lowrank}$ is a submersion.\footnote{A smooth mapping $\varphi \colon \mathcal{X} \to \mathcal{Y}$ between manifolds is called a \emph{submersion} at a point $x \in \mathcal{X}$ if its derivative $d\varphi_x \colon T_x(\mathcal{X}) \to T_{\varphi(x)}(\mathcal{Y})$ is surjective.  It is called a \emph{submersion} (unqualified) if it is a submersion at \emph{every} $x \in \mathcal{X}$.}
\end{lemma}

\begin{proof}
Let $\MQ = \begin{bmatrix} \St & \MV \end{bmatrix} \in \SO{\lowrank}$, $\St = \pi(\MQ)$, and $\dot{\MY} \in T_{\St}(\Stiefel{d}{\lowrank})$; we must show that there is some $\dot{\MX} \in T_{\MQ}(\SO{\lowrank})$ such that $\dot{\MY} = d\pi_{\MQ}(\dot{\MX})$.  Since $\St\tran \MV = 0$ (as the columns of $\MQ$ are orthonormal), by \eqref{Stiefel_manifold_tangent_spaces} there exist some $\dot{\MOmega} \in \Skew(d)$ and $\dot{\MK} \in \Real{(\lowrank -d) \times d}$ such that $\dot{\MY} = \St \dot{\MOmega} + \MV\dot{\MK}$.  Letting $\dot{\MGamma} \in \Skew(\lowrank - d)$, and defining
\begin{equation}
\dot{\MX} = \MQ
\begin{bmatrix}
\dot{\MOmega} & -\dot{\MK}\tran\\
\dot{\MK} & \dot{\MGamma}
\end{bmatrix} \in T_{\MQ}(\SO{\lowrank})
\end{equation}
as in \eqref{partitioned_rotation_manifold_tangent_spaces}, it follows from \eqref{derivative_of_rotation_to_Stiefel_projection} that $\dot{\MY} = d\pi_{\MQ}(\dot{\MX})$.
\end{proof}

\begin{corollary}
 The projection $\Pi \colon \SO{\lowrank}^\nrNodes \to \Stiefel{d}{\lowrank}^\nrNodes$ is a submersion.
 \label{Projection_submersion_corollary}
\end{corollary}

\noindent  \textbf{Proof of Theorem \ref{certifying_optimality_of_lifted_optimization_theorem}}:

Part(i): Applying the Chain Rule to \eqref{pullback_objective} produces:
\begin{equation}
\label{derivative_of_lifted_objective}
 d\tilde{f}_{\MQ^\star} = df_{\St^\star} \circ d\Pi_{\MQ^\star}.
\end{equation}
The left-hand side of \eqref{derivative_of_lifted_objective} is $0$ because $\MQ^\star$ is a first-order critical point of $\tilde{f}$ by hypothesis, and Corollary \ref{Projection_submersion_corollary} shows that $d\Pi_{\MQ^\star}$ is full-rank (i.e., its image is all of $T_{\St^\star}(\Stiefel{d}{\lowrank})^n$).  It follows that $df_{\St^\star} = 0$, and therefore $\St^\star$ is a first-order critical point of \eqref{eq:Riemannian_LRSDP}.

Part (ii):  Observe that $\Pi(\MQ^+) = \St^+$ as defined in Theorem \ref{Splus_definition}(ii); since $\St^+$ is a stationary point for \eqref{eq:Riemannian_LRSDP}, another application of the Chain Rule as in \eqref{derivative_of_lifted_objective} shows that $d\tilde{f}_{\MQ^+} = 0$, and therefore $\MQ^+$ is also a first-order critical point of the lifted optimization \eqref{eq:Shonan} in dimension $\lowrank + 1$.  Theorem \ref{saddle_escape_theorem}(ii) also provides the second-order descent direction $\dot{\St}^+$ from $\St^+$.  The tangent vector $\dot{\MQ}^+$ defined in \eqref{second_order_descent_direction_for_rotation_manifold} is constructed as a \emph{lift} of $\dot{\St}^+$ to $T_{\MQ^+}(\SO{\lowrank + 1}^\nrNodes)$; that is, so that $\dot{\MQ}^+$ satisfies $d\Pi_{\MQ^+}(\dot{\MQ}^+) = \dot{\St}^+$. Indeed, using \eqref{higher_dimensional_rotation}, \eqref{second_order_descent_direction_for_rotation_manifold}, and \eqref{derivative_of_rotation_to_Stiefel_projection} we can compute each block of $d\Pi_{\MQ^+}(\dot{\MQ}^+)$ as:
\begin{equation}
\begin{split}
d\Pi_{\MQ^+}(\dot{\MQ}^+)_{i} 
&= d\pi_{\MQ_i^+}\left(
\begin{bmatrix}
\MQ^\star_i & 0 \\
0 & 1
\end{bmatrix}
\begin{bmatrix}
0 & - v_i \\
v_i\tran & 0
\end{bmatrix}
\right)\\
&=
\begin{bmatrix}
0 \\
v_i\tran
\end{bmatrix}
= \dot{\St}_i^+
\end{split}
\end{equation}
for all $i \in [\nrNodes]$, verifying that $\dot{\MQ}^+$ is a lift of $\dot{\St}^+$.  

Our goal now is to show that $\dot{\MQ}^+$ is likewise a second-order descent direction from $\MQ^+$. We do this using a proof technique similar to that of Proposition 5.5.6 in \cite{Absil07book}. Let $R_{\MQ^+} \colon T_{\MQ^+}(\SO{\lowrank+1}^\nrNodes) \to \SO{\lowrank + 1}^\nrNodes$ be any (first-order) retraction on the tangent space of $\SO{\lowrank + 1}^\nrNodes$ at $\MQ^+$, and let $\epsilon > 0$ be sufficiently small that the curve:
\begin{equation}
\begin{gathered}
 \tilde{\gamma} \colon (-\epsilon, \epsilon) \to \SO{\lowrank + 1}^\nrNodes \\
 \tilde{\gamma}(t) = R_{\MQ^+}(t \dot{\MQ}^+)
 \end{gathered}
\end{equation}
obtained by moving through the point $\MQ^+ \in \SO{\lowrank + 1}^\nrNodes$ along the direction $\dot{\MQ}^+$ is well-defined.  Our approach simply involves examining the behavior of the lifted objective $\tilde{f}$ at points along the curve $\tilde{\gamma}$ in a neighborhood of $\tilde{\gamma}(0) = \MQ^+$.  To do so, we define:
\begin{equation}
\begin{gathered}
\tilde{f} \colon (-\epsilon, \epsilon) \to \Real{} \\
\tilde{f}(t) = \tilde{f} \circ \tilde{\gamma}(t)
\end{gathered}
\end{equation}
and then consider its first- and second-order derivatives.  Once again using the Chain Rule (and unwinding the definition of $\tilde{f}$), we compute:
\begin{equation}
\begin{split}
\tilde{f}'(t) &= d\tilde{f}_{\tilde{\gamma}(t)} \circ \tilde{\gamma}'(t) \\
&=  df_{\Pi \circ \tilde{\gamma}(t)} \circ d\Pi_{\tilde{\gamma}(t)} \circ \tilde{\gamma}'(t) \\
&= \left \langle \grad f \left( \Pi \circ \tilde{\gamma}(t)\right) , \: d\Pi_{\tilde{\gamma}(t)} \circ \gamma'(t)   \right \rangle.
\end{split}
\label{first_order_derivative_of_escape_curve}
\end{equation}
Note that at $t = 0$, $\grad f \left( \Pi \circ \tilde{\gamma}(0)\right) =  \grad f(\Pi(\MQ^+)) = \grad f(\St^+) = 0$ since $\St^+$ is a stationary point.  It follows from \eqref{first_order_derivative_of_escape_curve} that $f'(0) = 0$ (as expected, since we know that $\MQ^+$ is a stationary point for the lifted optimization \eqref{eq:Shonan}).  Continuing, we compute the second derivative of $\tilde{f}(t)$ by applying the Product and Chain Rules to differentiate the inner product on the final line of \eqref{first_order_derivative_of_escape_curve}:
\begin{equation}
\begin{split}
\tilde{f}''(t) &= \frac{d}{dt} \left[ \left \langle \grad f \left( \Pi \circ \tilde{\gamma}(t)\right) , \: d\Pi_{\tilde{\gamma}(t)} \circ \gamma'(t)   \right \rangle \right] \\
&= \left \langle \frac{d}{dt} \left[\grad f \left( \Pi \circ \tilde{\gamma}(t)\right) \right], \: d\Pi_{\tilde{\gamma}(t)} \circ \gamma'(t) \right \rangle \\
&\quad + \left \langle  \grad f \left( \Pi \circ \tilde{\gamma}(t)\right), \: \frac{d}{dt} \left[ d\Pi_{\tilde{\gamma}(t)} \circ \gamma'(t) \right] \right \rangle.
\end{split}
\label{setup_for_second_derivative_computation}
\end{equation}
Now, we just saw that $\grad f \left( \Pi \circ \tilde{\gamma}(0)\right) = 0$, so the second term of the final line of \eqref{setup_for_second_derivative_computation} is zero for $t = 0$.  Moreover, the derivative in the first term can be further developed as:
\begin{equation}
\frac{d}{dt} \left[\grad f \left( \Pi \circ \tilde{\gamma}(t)\right) \right] = \Hess f\left( \Pi \circ \tilde{\gamma}(t)\right)\left[ d\Pi_{\tilde{\gamma}(t)} \circ \tilde{\gamma}'(t) \right].
\label{derivative_of_gradient_vector_field}
\end{equation}
Therefore, at $t = 0$ \eqref{setup_for_second_derivative_computation} simplifies as:
\begin{equation}
\begin{split}
\tilde{f}''(0) &= \left \langle \Hess f\left( \Pi \circ \tilde{\gamma}(0)\right)\left[ d\Pi_{\tilde{\gamma}(0)} \circ \tilde{\gamma}'(0)\right], \: d\Pi_{\tilde{\gamma}(0)} \circ \gamma'(0) \right \rangle \\
&= \left \langle \Hess f\left( \Pi(\MQ^+) \right)\left[ d\Pi_{\MQ^+}(\dot{\MQ}^+) \right], \: d\Pi_{\MQ^+}(\dot{\MQ}^+) \right \rangle \\
&= \left \langle \Hess f\left(\St^+\right)\left[\dot{\St}^+\right], \: \dot{\St}^+\right \rangle \\
&< 0
\end{split}
\label{proof_of_second_order_descent}
\end{equation}
where the final line of \eqref{proof_of_second_order_descent} follows from the fact that $\dot{\St}^+$ is a second-order direction of descent from $\St^+$.  We conclude from \eqref{proof_of_second_order_descent} that  $\dot{\MQ}^+$ is a second-order direction of descent from $\MQ^+$, as desired. \hfill $\square$

\section{Gauss-Newton for Shonan Averaging}
\label{GN_for_Shonan_Averaging_section}

We can implement the local search for (first-order) critical points of \eqref{eq:Shonan} required in line \ref{local_search_in_Shonan_averaging} of the Shonan Averaging algorithm using the same Gauss-Newton approach described in \prettyref{sec:Synchronization}.

\subsection{Linearization}

As before, we first rewrite \eqref{eq:Shonan} more explicitly as the minimization of the sum of the individual measurement residuals, in a vectorized form:
\bea
\min_{ \QSOdN{\lowrank} } \sumalledges 
\kappa_{ij} \normsq{ \myvec(\MQ_j \MP  - \MQ_i \MP \barMR_{ij}) }{2},
\label{eq:GN1}
\eea
and reparameterize this minimization in terms in terms of the Lie algebra $\so(\lowrank)$:
\bea
\min_{\vdelta \in \so(\lowrank)^\nrNodes} \sumalledges 
\kappa_{ij} \normsq{ \myvec \left(\MQ_j\exphat{\vdelta_j} \MP
- \MQ_i\exphat{\vdelta_i} \MP \MR_{ij}\right)}{2}.
\label{eq:GN2}
\eea
Once again, we approximate \eqref{eq:GN2} to first order as the linear least-squares objective:
\bea 
\min_{\vdelta \in \so(\lowrank)^\nrNodes} \sumalledges 
\kappa_{ij} \normsq{ \MF_j \vdelta_j - \MH_i \vdelta_i - \vb_{ij}  }{2},
\label{GN2_over_Lie_algebra}
\eea
where now the Jacobians $\MF_j$ and $\MF_i$ and the right-hand side $\vb_{ij}$ are calculated as
\bea
\MF_j & \doteq & (\MP\tran \otimes \MQ_j) \bar{\MG}_\lowrank \\
\MH_i & \doteq & ((\MP \barMR_{ij})\tran \otimes \MQ_i) \bar{\MG}_\lowrank \\
\vb_{ij} & \doteq & \myvec (\MQ_i \MP \barMR_{ij} - \MQ_j \MP)
\label{rhs}
\eea
with $\bar{\MG}_\lowrank$ the matrix of vectorized $\so(\lowrank)^\nrNodes$ generators.
Note that the right-hand side $\vb_{ij}$ in \eqref{rhs} in fact involves only the Stiefel manifold elements $\St_i = \pi(\MQ_i)$:
\bea
\vb_{ij} = \myvec (\St_i \barMR_{ij} - \St_j).
\eea

\subsection{The structure of the Lie algebra}
In this section we investigate the structure of the Lie algebra $\so(\lowrank)$ as it pertains to the linearized objective in \eqref{GN2_over_Lie_algebra}.  Recall that $\so(\lowrank)$ is identified with the tangent space $T_{\eye_\lowrank}(\SO{\lowrank}) = \Skew(\lowrank)$, the space of $\lowrank \times \lowrank$ skew-symmetric matrices.  Once again writing these matrices in a block-partitioned form compatible with the projection $\pi$ in \eqref{rotation_to_Stiefel_projection} (as in \eqref{partitioned_rotation_manifold_tangent_spaces}) produces:
\begin{equation}
\so(\lowrank) = \left \lbrace
\begin{bmatrix}
[\vomega] & -\MK\tran \\
\MK & [\vgamma] 
\end{bmatrix}
\mid 
\begin{aligned}
\vomega &\in \so(d), \\
\MK &\in \Real{(\lowrank -d) \times d}, \\
\vgamma &\in \so(\lowrank-d)
\end{aligned}
\right \rbrace
\label{Tangent_space_of_rotations_at_the_identity},
\end{equation}
and it follows from \eqref{rotation_to_Stiefel_projection} the derivative of $\pi$ at $\eye_\lowrank$ is:
\begin{equation}
\begin{gathered}
d\pi_{\eye_\lowrank}\left( 
\begin{bmatrix}
[\vomega] & -\MK\tran \\
\MK & [\vgamma] 
\end{bmatrix}
\right) = 
\begin{bmatrix}
[\vomega] \\
\MK\end{bmatrix}.
\label{derivative_of_pi_at_Ip}
\end{gathered}
\end{equation}

Note that $d\pi_{\eye_\lowrank}$ does \emph{not} depend upon the $(\lowrank - d)$-dimensional vector $\vgamma$.  In particular, let us define:
\begin{equation}
V_{\eye_\lowrank} \doteq \ker d\pi_{\eye_\lowrank} = \left \lbrace 
\begin{pmatrix}
0 & 0 \\
0 & [\vgamma]
\end{pmatrix} \mid \vgamma \in \so(\lowrank - d)
\right \rbrace \subset \so(\lowrank).
\label{vertical_space}
\end{equation}
Geometrically, the set $V_{\eye_\lowrank}$ defined in \eqref{vertical_space} consists of those directions of motion $\dot{\MOmega} \in \so(\lowrank)$ at $\eye_\lowrank$ along which the projection $\pi$ is \emph{constant}.  Equivalently:
\begin{equation}
V_{\eye_\lowrank} = T_{\eye_\lowrank}( \pi^{-1}(\MP));
\label{vertical_space_as_tangent_space_of_fiber}
\end{equation}
that is, $V_{\eye_\lowrank}$ is the set of vectors that are \emph{tangent} to the preimage $\pi^{-1}(\MP)$ of $\MP = \pi(\eye_\lowrank)$ at $\eye_\lowrank$.  If we think of the elements of the preimage $\pi^{-1}(\MP)$ as being vertically ``stacked" above their common projection $\MP \in \Stiefel{d}{\lowrank}$, then the subspace $V_{\eye_\lowrank}$ of the Lie algebra is precisely the set of tangent vectors at $\eye_\lowrank$ that correspond to \emph{vertical motions}.  Consequently, $V_{\eye_\lowrank}$ is referred to as the \emph{vertical space}.  We may define a corresponding \emph{horizontal space} in the natural way, i.e., as the orthogonal complement of the vertical space:
\begin{equation}
H_{\eye_\lowrank} \doteq V_{\eye_\lowrank}^{\perp}  = 
\left \lbrace 
\begin{pmatrix}
[\vomega] & -\MK\tran \\
\MK & 0
\end{pmatrix} \mid 
\begin{aligned}
\vomega &\in \so(d), \\
\MK &\in \Real{(\lowrank -d) \times d}
\end{aligned}
\right\rbrace \subset \so(\lowrank).
\label{horizontal_space}
\end{equation}

%
The significance of \eqref{Tangent_space_of_rotations_at_the_identity}--\eqref{horizontal_space} is that, to first order, the exponential map (or any retraction) can be written as $\eye+\liehat{\vdelta}$ for $\vdelta \in \so(\lowrank$).  In conjunction with the projection map $\pi$ (equivalently, with $\MP$), this implies:
\bea
\MQ_i \exphat{\vdelta}\MP \approx
\MQ_i (\eye+\liehat{\vdelta})\MP =
{\St}_i + \MQ_i \begin{bmatrix} \liehat{\vomega} \\ \MK \end{bmatrix}.
\eea
From this we see that the derivative of the cost function \eqref{eq:GN2} will not depend on the $(\lowrank-d)$-dimensional vector $\vgamma$, i.e., on the component of $[\vdelta]$ lying in the vertical space $V_{\eye_\lowrank}$.  This makes intuitive sense, since the Shonan Averaging objective $\tilde{f}(\MQ)$ in \eqref{eq:Shonan} is defined in terms of the projection $\Pi(\MQ)$ (recall \eqref{pullback_objective}), and moving along vertical directions leaves this projection unchanged.  

This in turn enables us to characterize the Jacobians in more detail. If we split $\MQ_i=\begin{bmatrix}\St_i & \MV_i\end{bmatrix}$, the Jacobians $\MF_j$ and $\MH_i$ can be shown to be:
\bea
\MF_j & \doteq & 
\begin{bmatrix}
(\eye_d \otimes \St_j) \bar{\MG}_d & (\eye_d \otimes {\MV}_j) & 0 
\end{bmatrix} \\
\MH_i & \doteq &
\begin{bmatrix}
(\barMR_{ij}\tran \otimes {\St}_i) \bar{\MG}_d & (\barMR_{ij}\tran \otimes  {\MV}_i) & 0 
\end{bmatrix} 
\label{eq:detailed-jacobians}
\eea
where $\bar{\MG}_d$ is the matrix of vectorized generators for the Lie algebra $\so(d)$ (as in Section \ref{GaussNewton_for_SOd}).  Again we see that the last columns, corresponding to the vertical directions, are zero.

The astute reader may now wonder whether the rank-deficiency of these Jacobians poses any numerical difficulties when solving the linear systems needed to compute the update step $\vdelta$.  In fact there are several straightforward ways to address this.  One approach is simply to employ the Levenberg-Marquardt method directly in conjunction with the Jacobians \eqref{eq:detailed-jacobians}; in this case, the Tikhonov regularization applied by the LM algorithm itself will ensure that all of the linear systems to be solved are nonsingular.  Moreover, this regularization will additionally encourage update steps to lie in the horizontal subspace, since any update with a nonzero vertical component will have to ``pay" for the length of that component (via regularization), while ``gaining" nothing for it (in terms of reducing the local model of the objective). 

Alternatively, one can remove the final $(\lowrank - d)$ all-$0$'s columns from the Jacobians in \eqref{eq:detailed-jacobians}, and solve the resulting reduced linear system in the variables $(\vomega, \MK)$.  Geometrically, this corresponds to minimizing the local quadratic model of the objective \eqref{GN2_over_Lie_algebra} over the horizontal subspace; a horizontal full-space update $[\vdelta]$ can then be obtained by simply taking $\vgamma = 0$.  It is straightforward to see that this procedure corresponds to computing a pseudoinverse (minimum-norm) minimizer of the quadratic model \eqref{GN2_over_Lie_algebra}.

Finally, a third approach is to regularize the original Shonan Averaging problem \eqref{eq:Shonan} by adding a prior term on the Karcher mean of the rotations $\MQ_i$ in $\MQ$; this has the effect of fixing the gauge for the underlying estimation problem, similarly to the use of ``inner constraints"  in photogrammetry \cite{Triggs00}.



\newpage
\section{More Experimental Results on the YFCC Datasets}

In this section we present more extensive results on the datasets derived by Heinly et al.~\cite{Heinly15cvpr_ReconstructWorld6Days} from the large-scale YFCC-100M dataset~\cite{Thomee16acm_YFCC100M,Heinly15cvpr_ReconstructWorld6Days}. As in the main paper, the relative measurements for these were derived from the SFM solution provided with the data, and corrupted with noise as before, using $\sigma=0.2$.
For all results below, minimum, average, and maximum running times (in seconds) are computed over 10 random initializations for each dataset. Also shown is the fraction of cases in which the method converges to a global minimizer.
All Shonan Averaging variants examined below use the same Levenberg-Marquardt non-linear optimizer, with a Jacobi-preconditioned conjugate gradient method as the linear system solver.
 
\subsection{Small Datasets ($n<50$)}
\begin{table*}[h!]
\footnotesize
\centerline{
\begin{tabular}{lcrrrrr}
\hline
 dataset                        &  method  &   error &     min &     avg &     max &   success \\
\hline
\hline
statue\_of\_liberty\_1           &    SA    &  0.000\% &   0.038 &   0.313 &   1.211 &      100\% \\
 (n=19, m=54)                   &    SL    &  0.000\% &   0.010 &   0.219 &   0.901 &      100\% \\
                                &    S3    &    nan\% & nan     & nan     & nan     &        0\% \\
                                &    S4    &  0.001\% &   0.009 &   0.016 &   0.019 &       30\% \\
                                &    S5    &  0.001\% &   0.010 &   0.018 &   0.022 &       40\% \\
                                &    SK    & -0.000\% &   0.008 &   \textbf{0.108} &   0.459 &      100\% \\
\hline
natural\_history\_museum\_london &    SA    &  0.000\% &   0.019 &   0.036 &   0.049 &      100\% \\
 (n=30, m=274)                  &    SL    &  0.000\% &   0.011 &   \textbf{0.021} &   0.068 &      100\% \\
                                &    S3    &  0.000\% &   0.010 &   0.013 &   0.015 &       60\% \\
                                &    S4    &  0.000\% &   0.011 &   0.016 &   0.021 &      100\% \\
                                &    S5    &  0.000\% &   0.009 &   0.014 &   0.018 &      100\% \\
                                &    SK    & -0.000\% &   0.021 &   \textbf{0.022} &   0.024 &      100\% \\
\hline
statue\_of\_liberty\_2           &    SA    &  0.000\% &   0.030 &   0.063 &   0.094 &      100\% \\
 (n=39, m=156)                  &    SL    &  0.001\% &   0.011 &   \textbf{0.034} &   0.060 &      100\% \\
                                &    S3    &  0.000\% &   0.010 &   0.028 &   0.046 &       40\% \\
                                &    S4    &  0.000\% &   0.011 &   0.024 &   0.047 &      100\% \\
                                &    S5    &  0.000\% &   0.010 &   0.030 &   0.057 &      100\% \\
                                &    SK    &  0.000\% &   0.019 &   0.050 &   0.113 &      100\% \\
\hline
taj\_mahal\_entrance            &    SA    &  0.000\% &   0.071 &   0.117 &   0.165 &      100\% \\
 (n=42, m=1272)                 &    SL    &  0.000\% &   0.032 &   \textbf{0.043} &   0.062 &      100\% \\
                                &    S3    &  0.000\% &   0.037 &   0.046 &   0.063 &       80\% \\
                                &    S4    &  0.000\% &   0.033 &   0.042 &   0.051 &      100\% \\
                                &    S5    &  0.000\% &   0.032 &   0.037 &   0.039 &      100\% \\
                                &    SK    & -0.000\% &   0.070 &   0.081 &   0.092 &      100\% \\
\hline
sistine\_chapel\_ceiling\_1      &    SA    &  0.000\% &   0.102 &   0.173 &   0.246 &      100\% \\
 (n=49, m=1754)                 &    SL    &  0.000\% &   0.064 &   \textbf{0.085} &   0.108 &      100\% \\
                                &    S3    &  0.000\% &   0.057 &   0.073 &   0.087 &       60\% \\
                                &    S4    &  0.000\% &   0.072 &   0.121 &   0.293 &      100\% \\
                                &    S5    &  0.000\% &   0.064 &   0.083 &   0.095 &      100\% \\
                                &    SK    & -0.000\% &   0.102 &   0.116 &   0.130 &      100\% \\
\hline
\end{tabular}
}
\caption{More results on YFCC datasets with $n < 50$. In this table, we compare 6 methods (see text for details). For each, we show the relative error with respect to SA, we give the minimum, average, and maximum running times (in seconds), and the fraction of cases in which the method converges to a global minimizer.} 
\label{YFCC50}
\end{table*}

Table \ref{YFCC50} shows additional experimental results on the small YFCC datasets (with $n<50$) with a more systematic exploration of the Shonan parameters. The intent is to provide more quantitative results for Shonan Averaging and its convergence properties at different levels of $p$, as well as explore parameter settings for optimal performance in practical settings.

In particular, we compared
\begin{itemize}
    \item SA: Shonan Averaging with $p_{min}=5$ and $p_{max}=30$.
    \item BD: the block coordinate descent method from~\cite{Eriksson18cvpr-strongDuality,Eriksson19pami_duality}. 
    \item SL: Same as SA, but starting (L)ow, from $SO(3)$: $p_{min}=3$ and $p_{max}=30$.
    \item S3: Only run with $p=3$, which corresponds to LM in the main paper: we only optimize at the base $SO(3)$ level.  Note that this approach has no guarantee of converging to a global minimizer.
    \item S4: Similar to S3, but with $P=4$: assesses in what percentage of cases we converge to global minima for $p=4$.
    \item S5: Similar to S3 and S4, but for $p=5$.
    \item SK: Shonan Averaging with $p_{min}=5$ and $p_{max}=30$, i.e., the same as SA, but with a different prior to fix the gauge freedom.
\end{itemize}

The last version, SK, was inspired by~\cite{Wilson2016rotations}, who stressed the importance of fixing the gauge symmetry to make the rotation averaging problem better-behaved. In particular, we fixed the Karcher mean of all rotations (for any level $p$) to remain at its initial value, similar to the ``inner constraints" often used in photogrammetry \cite{Triggs00}.

In Table \ref{YFCC50}, we have indicated the best performing method out of SA, SL, and SK. The specialized solvers S3, S4, and S5 that only optimize at one level are not considered in the comparison, as they are not guaranteed to converge to a global minimum, and are only there to provide insight into the relative performance in those different $SO(p)$ spaces.

\paragraph{Conclusions} For these small datasets, block-coordinate descent~\cite{Eriksson18cvpr-strongDuality,Eriksson19pami_duality} performs very well. Even so, Shonan averaging with Karcher mean is faster in several cases. One such example in this Table is the Statue of Liberty dataset with $n=39$ and $m=156$.  Interestingly, SK appears to be the fastest Shonan Averaging variant for some datasets, despite the fact that it contains an additional \textit{dense} term in the objective that involves \emph{all} of the poses.  We conjecture that for these relatively small datasets, the inclusion of the prior on the Karcher mean helps to promote faster convergence of the manifold optimization by penalizing components of the update step that lie in the subspace of (global) gauge symmetry directions for the rotation averaging problem. Intuitively, it discourages the step from having a component that does not ``actually change" the solution.  
This can be important in the context of a trust-region method like ours, where the \emph{total} length of the step is restricted at each iteration.  It would also be interesting to investigate what (if any) effect the inclusion of the Karcher mean has on the presence of suboptimal critical points at each level of the Riemannian Staircase, although we leave these questions for future work. 


Also clear is that starting Shonan Averaging with $p=3$, shown as SL in the table, is always either on par or much faster than SA. There is a simple explanation for this: in many instances, it is possible to recover a global minimizer from the optimization at the lowest level $p=3$. This can be appreciated by comparing with the results of S3: it rarely finds global minima \textit{every} time, but when it does it is obviously the fastest of all methods. In SL, we only move to the next $SO(p)$ level if that does \emph{not} happen, and \textbf{hence we get the best of both worlds: fast convergence if we happened to pick a lucky initial estimate, and upgrade to global optimality if not}. 

The S3, S4, and S5 lines are shown to indicate at what level this occurs, and for these datasets it is almost always at $p=4$. However, the results reveal that there are indeed no guarantees, so it is not a good idea to run Shonan Averaging at a single level: the Riemmannian Staircase provides the global guarantee but at minimal extra cost, as it is \emph{only} triggered when we converge to a suboptimal critical point at a lower level $p$.

\subsection{Intermediate-size Datasets ($50 \leq n < 150$)}
\begin{table*}[h!]
\footnotesize
\centerline{
\begin{tabular}{lcrrrrr}
\hline
 dataset                      &  method  &   error &   min &   avg &   max &   success \\
\hline
\hline
sistine\_chapel\_ceiling\_2    &    SA    &  0.000\% & 0.101 & 0.162 & 0.234 &      100\% \\
 (n=51, m=1670)               &    SL    &  0.000\% & 0.065 & \textbf{0.115} & 0.332 &      100\% \\
                              &    S3    &  0.000\% & 0.076 & 0.081 & 0.084 &       50\% \\
                              &    SK    & -0.000\% & 0.097 & 0.127 & 0.211 &      100\% \\
\hline
milan\_cathedral             &    SA    &  0.000\% & 0.175 & 0.203 & 0.220 &      100\% \\
 (n=69, m=2782)               &    SL    &  0.000\% & 0.078 & \textbf{0.092} & 0.105 &      100\% \\
                              &    SK    & -0.000\% & 0.170 & 0.190 & 0.206 &      100\% \\
\hline
reichstag                   &    SA    &  0.000\% & 0.166 & 0.291 & 0.399 &      100\% \\
 (n=71, m=2554)               &    SL    &  0.000\% & 0.073 & \textbf{0.084} & 0.108 &      100\% \\
                              &    SK    & -0.000\% & 0.179 & 0.193 & 0.215 &      100\% \\
\hline
piazza\_dei\_miracoli         &    SA    &  0.000\% & 0.261 & 0.353 & 0.593 &      100\% \\
 (n=74, m=3456)               &    SL    &  0.000\% & 0.092 & \textbf{0.119} & 0.147 &      100\% \\
                              &    S3    &  0.000\% & 0.102 & 0.119 & 0.146 &       90\% \\
                              &    SK    & -0.000\% & 0.260 & 0.344 & 0.413 &      100\% \\
\hline
ruins\_of\_st\_pauls           &    SA    &  0.000\% & 0.337 & 0.426 & 0.735 &      100\% \\
 (n=82, m=4998)               &    SL    &  0.000\% & 0.149 & \textbf{0.173} & 0.259 &      100\% \\
                              &    SK    & -0.000\% & 0.372 & 0.450 & 0.544 &      100\% \\
\hline
mount\_rushmore              &    SA    &  0.000\% & 0.264 & 0.366 & 0.511 &      100\% \\
 (n=83, m=4012)               &    SL    &  0.000\% & 0.105 & \textbf{0.123} & 0.193 &      100\% \\
                              &    SK    & -0.000\% & 0.263 & 0.306 & 0.352 &      100\% \\
\hline
london\_bridge\_3             &    SA    &  0.000\% & 0.113 & 0.141 & 0.243 &      100\% \\
 (n=88, m=1500)               &    SL    &  0.000\% & 0.050 & \textbf{0.070} & 0.095 &      100\% \\
                              &    S3    &  0.000\% & 0.052 & 0.066 & 0.084 &       90\% \\
                              &    SK    & -0.000\% & 0.115 & 0.128 & 0.143 &      100\% \\
\hline
palace\_of\_versailles\_chapel &    SA    &  0.000\% & 0.296 & 0.319 & 0.379 &      100\% \\
 (n=91, m=3964)               &    SL    &  0.000\% & 0.102 & \textbf{0.118} & 0.138 &      100\% \\
                              &    SK    & -0.000\% & 0.291 & 0.313 & 0.348 &      100\% \\
\hline
pieta\_michelangelo          &    SA    &  0.000\% & 0.540 & 0.709 & 1.051 &      100\% \\
 (n=93, m=7728)               &    SL    &  0.000\% & 0.194 & \textbf{0.220} & 0.228 &      100\% \\
                              &    SK    & -0.000\% & 0.538 & 0.559 & 0.614 &      100\% \\
\hline
blue\_mosque\_interior\_2      &    SA    &  0.000\% & 0.169 & 0.220 & 0.348 &      100\% \\
 (n=95, m=2288)               &    SL    &  0.000\% & 0.075 & 0.149 & 0.510 &      100\% \\
                              &    S3    &  0.000\% & 0.087 & \textbf{0.106} & 0.140 &       90\% \\
                              &    SK    & -0.000\% & 0.172 & 0.200 & 0.222 &      100\% \\
\hline
st\_vitus\_cathedral          &    SA    &  0.000\% & 0.592 & 0.909 & 1.306 &      100\% \\
 (n=97, m=8334)               &    SL    &  0.000\% & 0.250 & \textbf{0.291} & 0.326 &      100\% \\
                              &    S3    &  0.000\% & 0.252 & 0.298 & 0.387 &       90\% \\
                              &    SK    & -0.000\% & 0.571 & 0.650 & 0.805 &      100\% \\
\hline
\end{tabular}
}
\caption{YFCC results for $50 \leq n \leq 100$. Same format as Table \ref{YFCC50} but only showing S3-S5 results if they do not converge to global minima in \emph{all} tested cases.} 
\label{YFCC100}
\end{table*}

\begin{table*}[h!]
\footnotesize
\centerline{
\begin{tabular}{lcrrrrr}
\hline
 dataset                          &  method  &   error &   min &   avg &    max &   success \\
\hline
 \hline
big\_ben\_1                       &    SA    &  0.000\% & 0.196 & 4.084 & 30.563 &      100\% \\
 (n=101, m=1880)                  &    SL    &  0.000\% & 0.080 & \textbf{0.348} &  1.285 &      100\% \\
                                  &    S3    &  0.000\% & 0.105 & 0.127 &  0.199 &       70\% \\
                                  &    S4    &  0.000\% & 0.090 & 0.118 &  0.172 &       60\% \\
                                  &    S5    &  0.000\% & 0.077 & 0.140 &  0.197 &       80\% \\
                                  &    SK    &  0.000\% & 0.185 & 1.001 &  5.572 &      100\% \\
 \hline
london\_bridge\_2                 &    SA    &  0.000\% & 0.220 & 0.351 &  1.218 &      100\% \\
 (n=106, m=2742)                  &    SL    &  0.000\% & 0.091 & \textbf{0.114} &  0.136 &      100\% \\
                                  &    SK    & -0.000\% & 0.222 & 0.243 &  0.264 &      100\% \\
 \hline
palazzo\_pubblico                &    SA    &  0.000\% & 0.331 & 0.373 &  0.427 &      100\% \\
 (n=112, m=4420)                  &    SL    &  0.000\% & 0.155 & \textbf{0.262} &  0.926 &      100\% \\
                                  &    SK    & -0.000\% & 0.340 & 0.382 &  0.442 &      100\% \\
 \hline
london\_bridge\_1                 &    SA    &  0.000\% & 0.416 & 0.472 &  0.543 &      100\% \\
 (n=118, m=5690)                  &    SL    &  0.000\% & 0.168 & \textbf{0.203} &  0.285 &      100\% \\
                                  &    SK    & -0.000\% & 0.423 & 0.469 &  0.536 &      100\% \\
 \hline
national\_gallery\_london         &    SA    &  0.000\% & 0.186 & 0.233 &  0.351 &      100\% \\
 (n=124, m=2160)                  &    SL    &  0.000\% & 0.089 & \textbf{0.107} &  0.131 &      100\% \\
                                  &    SK    & -0.000\% & 0.179 & 0.210 &  0.246 &      100\% \\
 \hline
lincoln\_memorial                &    SA    &  0.000\% & 0.324 & 0.374 &  0.548 &      100\% \\
 (n=127, m=3516)                  &    SL    &  0.000\% & 0.119 & \textbf{0.135} &  0.185 &      100\% \\
                                  &    SK    & -0.000\% & 0.259 & 0.301 &  0.329 &      100\% \\
 \hline
grand\_central\_terminal\_new\_york &    SA    &  0.000\% & 0.414 & 0.584 &  1.043 &      100\% \\
 (n=132, m=5880)                  &    SL    &  0.000\% & 0.187 & \textbf{0.220} &  0.291 &      100\% \\
                                  &    SK    & -0.000\% & 0.528 & 0.586 &  0.661 &      100\% \\
 \hline
paris\_opera\_2                   &    SA    &  0.000\% & 0.809 & 0.888 &  0.982 &      100\% \\
 (n=133, m=10778)                 &    SL    &  0.000\% & 0.341 & \textbf{0.395} &  0.478 &      100\% \\
                                  &    SK    & -0.000\% & 0.837 & 0.906 &  1.058 &      100\% \\
\hline
\end{tabular}
}
\caption{YFCC results for $100 \leq n \leq 150$. Same layout as Table \ref{YFCC100}.} 
\label{YFCC150}
\end{table*}

In Tables \ref{YFCC100} and \ref{YFCC150} we show additional results on increasingly larger YFCC datasets, with exactly the same parameters as in the previous section. However, here we omit the S3-S5 variants \textit{unless} they do not converge to global minima in \emph{all} tested cases.

\paragraph{Conclusion} From these results \textbf{it is clear that SL starts to emerge as the best among the global optimization methods in terms of average running time}. Again, we observe that for these larger datasets it is rare \emph{not} to converge to a global minimizer at $p=3$, which is interesting in its own right. Of course, there are some exceptions, e.g., the Big Ben dataset in Table \ref{YFCC150} with $n=101$ and $m=1880$, for which global minimizers were \emph{not} found using only local search with $p \in \lbrace 3, 4, 5 \rbrace$.
The BD method is still competitive in cases where the number of measurements is large, as BD's running time is dominated by the number of images $n$, given that it is optimizing for a $3n\times 3n$ PSD matrix.

Results on these larger datasets suggest that finding global minima is actually \emph{harder} in general for small datasets than for larger, well connected datasets.
We can gain some theoretical insight into this empirical finding in light of several recent works \cite{Wilson2016rotations,Rosen16wafr-sesync,Boumal14ii} that have studied the  connection between graph-theoretic properties of the measurement network  $G = (n, \mathcal{E})$ that underpins the rotation averaging problem, and the statistical and geometric/computational properties of the resulting maximum-likelihood estimation  \eqref{eq:RA0}.  In a nutshell, these investigations indicate that \emph{both} the statistical properties of the maximum-likelihood estimation \eqref{eq:RA0} \emph{and} its computational hardness are controlled by the \emph{algebraic connectivity} $\lambda_2(L(G))$, i.e., the \emph{second}-smallest eigenvalue of the weighted graph Laplacian $L(G)$ associated with the graph $G$. Larger values imply \emph{both} a better (lower-uncertainty) estimate \emph{and} that the resulting relaxation \eqref{eq:SDP} is stronger.  For densely-connected measurement networks (of the kind that frequently appear in structure-from-motion applications), it is an elementary result from algebraic graph theory that this quantity can grow at a rate of up to $O(n)$. This provides insight into the observation that problems with larger measurement networks appear easier to solve, assuming a reasonably dense set of measurements.

\subsection{Larger Datasets ($n \geq 150$)}
\begin{table*}[h!]
\footnotesize
\centerline{
\begin{tabular}{lcrrrrr}
\hline
 dataset                             &  method  &   error &   min &   avg &    max &   success \\
\hline
\hline
st\_peters\_basilica\_interior\_2      &    SA    &  0.000\% & 0.948 & 1.165 &  1.738 &      100\% \\
 (n=173, m=11688)                    &    SL    &  0.000\% & 0.557 & \textbf{0.713} &  0.897 &      100\% \\
                                     &    S3    &  0.000\% & 0.531 & 0.587 &  0.715 &       80\% \\
\hline
pantheon\_interior                  &    SA    &  0.000\% & 0.852 & 1.001 &  1.570 &      100\% \\
 (n=186, m=10000)                    &    SL    &  0.000\% & 0.449 & \textbf{0.509} &  0.623 &      100\% \\
                                     &    S3    &  0.000\% & 0.455 & 0.530 &  0.642 &      100\% \\
\hline
florence\_cathedral\_dome\_interior\_1 &    SA    &  0.000\% & 2.435 & 2.848 &  3.226 &      100\% \\
 (n=213, m=31040)                    &    SL    &  0.000\% & 1.869 & \textbf{2.170} &  2.499 &      100\% \\
                                     &    S3    &  0.000\% & 1.856 & 2.209 &  2.814 &      100\% \\
\hline
paris\_opera\_1                      &    SA    &  0.000\% & 3.673 & 4.145 &  4.467 &      100\% \\
 (n=254, m=45754)                    &    SL    &  0.000\% & 1.686 & \textbf{2.024} &  2.677 &      100\% \\
                                     &    S3    &  0.000\% & 1.683 & 2.104 &  2.510 &      100\% \\
\hline
pike\_place\_market                  &    SA    &  0.000\% & 4.437 & 5.255 &  7.285 &      100\% \\
 (n=265, m=53242)                    &    SL    &  0.000\% & 2.019 & \textbf{2.322} &  3.400 &      100\% \\
                                     &    S3    &  0.000\% & 2.122 & 2.246 &  2.384 &      100\% \\
\hline
blue\_mosque\_interior\_1             &    SA    &  0.000\% & 3.581 & 4.153 &  4.964 &      100\% \\
 (n=272, m=40292)                    &    SL    &  0.000\% & 2.268 & \textbf{2.708} &  2.991 &      100\% \\
                                     &    S3    &  0.000\% & 2.511 & 2.997 &  3.645 &      100\% \\
\hline
notre\_dame\_rosary\_window           &    SA    &  0.000\% & 8.102 & 8.652 & 10.515 &      100\% \\
 (n=326, m=93104)                    &    SL    &  0.000\% & 3.478 & \textbf{4.074} &  4.690 &      100\% \\
                                     &    S3    &  0.000\% & 3.885 & 4.326 &  5.012 &      100\% \\
\hline
british\_museum                     &    SA    &  0.000\% & 4.297 & 4.851 &  7.846 &      100\% \\
 (n=344, m=45450)                    &    SL    &  0.000\% & 1.927 & \textbf{2.558} &  3.400 &      100\% \\
                                     &    S3    &  0.000\% & 2.009 & 2.558 &  3.911 &      100\% \\
\hline
palace\_of\_westminster              &    SA    &  0.000\% & 1.334 & 1.480 &  1.692 &      100\% \\
 (n=345, m=11522)                    &    SL    &  0.000\% & 0.598 & \textbf{0.837} &  1.242 &      100\% \\
                                     &    S3    &  0.000\% & 0.731 & 1.038 &  2.397 &      100\% \\
\hline
louvre                             &    SA    &  0.000\% & 2.872 & 3.141 &  3.709 &      100\% \\
 (n=367, m=26656)                    &    SL    &  0.000\% & 1.606 & \textbf{2.284} &  3.265 &      100\% \\
                                     &    S3    &  0.000\% & 1.648 & 2.058 &  2.495 &      100\% \\
\hline
st\_peters\_basilica\_interior\_1      &    SA    &  0.000\% & 5.283 & 5.654 &  6.295 &      100\% \\
 (n=365, m=55024)                    &    SL    &  0.000\% & 3.255 & \textbf{4.030} &  4.995 &      100\% \\
                                     &    S3    &  0.000\% & 3.662 & 4.169 &  4.728 &      100\% \\
\hline
st\_pauls\_cathedral                 &    SA    &  0.000\% & 7.385 & 7.861 &  8.384 &      100\% \\
 (n=370, m=83060)                    &    SL    &  0.000\% & 3.385 & \textbf{3.792} &  4.373 &      100\% \\
                                     &    S3    &  0.000\% & 3.420 & 4.176 &  6.689 &      100\% \\
\hline
westminster\_abbey\_1                &    SA    &  0.000\% & 4.238 & 4.582 &  4.957 &      100\% \\
 (n=501, m=38863)                    &    SL    &  0.000\% & 2.059 & \textbf{2.396} &  2.710 &      100\% \\
                                     &    S3    &  0.000\% & 2.209 & 2.693 &  3.385 &      100\% \\
\hline
pantheon\_exterior                  &    SA    &  0.000\% & 6.922 & 7.876 & 11.202 &      100\% \\
 (n=720, m=49256)                    &    SL    &  0.000\% & 3.798 & \textbf{5.247} &  6.765 &      100\% \\
                                     &    S3    &  0.000\% & 4.091 & 4.945 &  6.156 &      100\% \\
\hline
\end{tabular}
}
\caption{Comparing SA, SL and S3 (see text) on YFCC Datasets with $n \geq 150$.  In many cases S3 has a lower average computation time than SL, since it performs only local optimization at the \emph{lowest} level $p = 3$ of the Riemannian Staircase. However, in contrast to SL, S3 has \emph{no} guarantees regarding convergence to global minima.} 
\label{YFCCbig}
\end{table*}

Finally, in Table \ref{YFCCbig} we show additional results on the largest YFCC datasets with $n\geq 150$.
The block-coordinate descent method from ~\cite{Eriksson19pami_duality} did not converge in reasonable time for many of the larger datasets, which is because we use the minimum eigenvalue optimality certificate threshold $\lambda_{min}$ to establish convergence. The threshold we used in all experiments was $10^{-5}$, and for these datasets it takes a long time for BD to reach that level, in contrast to Shonan Averaging. 

\paragraph{Conclusion} For these large datasets convergence to the global minimizer occurs almost always at $p=3$ and hence the SL and S3 methods are basically identical in terms of operation and performance. However, SL comes with a global guarantee: in the rare case that S3 does not converge to the global minimizer at the $SO(3)$ level, SL will simply move up to $SO(4)$ and up, and thereby still recover the true maximum likelihood estimate.




\end{document}